\documentclass[twoside]{article}

\usepackage[accepted]{aistats2022}
\usepackage[round]{natbib}

\usepackage[utf8]{inputenc} % allow utf-8 input
\usepackage[T1]{fontenc}    % use 8-bit T1 fonts
\usepackage{hyperref}       % hyperlinks
\usepackage{url}            % simple URL typesetting
\usepackage{amsfonts}       % blackboard math symbols
\usepackage{nicefrac}       % compact symbols for 1/2, etc.
\usepackage{microtype}      % microtypography
\usepackage{amsmath, amssymb, amsthm}
\usepackage{epsfig}
\usepackage{wrapfig,lipsum,booktabs}
\usepackage{algorithm}
\usepackage{algorithmic}
\usepackage{graphicx}
\usepackage{subfigure}
\newcommand{\norm}[1]{\left\lVert#1\right\rVert}
\newcommand{\dd}{\mathrm{d}}
\newtheorem{thm}{Theorem}

\newtheorem{lem}{Lemma}

\begin{document}


\runningauthor{Xuepeng Shi, Pengfei Zheng, Adam Ding, Yuan Gao, Weizhong Zhang}

\twocolumn[

\aistatstitle{Finding Dynamics Preserving Adversarial Winning Tickets}
\vspace{-20pt}
\aistatsauthor{ Xupeng Shi$^{*1}$  \And Pengfei Zheng$^{*2}$  \And  A. Adam Ding$^1$ \And Yuan Gao$^{3}$ \And Weizhong Zhang$^{\dag 4}$ }
%\vspace{40pt}
\aistatsaddress{ } 
]

\begin{abstract}
  Modern deep neural networks (DNNs) are vulnerable to adversarial attacks and adversarial training has been shown to be a promising method for improving the adversarial robustness of DNNs. Pruning methods have been considered in adversarial context to reduce model capacity and improve adversarial robustness simultaneously in training. Existing adversarial pruning methods generally mimic the classical pruning methods for natural training, which follow the three-stage 'training-pruning-fine-tuning' pipelines. We observe that such pruning methods do not necessarily preserve the dynamics of dense networks, making it potentially hard to be fine-tuned to compensate the accuracy degradation in pruning. Based on recent works of \textit{Neural Tangent Kernel} (NTK), we systematically study the dynamics of adversarial training and prove the existence of trainable sparse sub-network at initialization which can be trained to be adversarial robust from scratch. This theoretically verifies the \textit{lottery ticket hypothesis} in adversarial context and we refer such sub-network structure as \textit{Adversarial Winning Ticket} (AWT). We also show empirical evidences that AWT preserves the dynamics of adversarial training and achieve equal performance as dense adversarial training. 
\end{abstract}
\section{Introduction}\label{sec:intro}
Deep neural networks (DNN) are widely used as the state-of-art machine learning classification systems due to its great performance gains in recent years. Meanwhile, as pointed out in \citet{Szegedy}, state-of-the-art DNN are usually vulnerable to attacks by \textit{adversarial examples}, inputs that are distinguishable to human eyes but can fool classifiers to make arbitrary predictions. Such undesirable property may prohibit DNNs from being applied to security-sensitive applications. Various of adversarial defense methods \citep{ExplainAd, distillation, DefenseGAN, AdvRobustMNIST, DistributionalRobustness} were then proposed to prevent adversarial examples attack. However, most of the defense methods were quickly broken by new adversarial attack methods. \textit{Adversarial training}, proposed in \citet{madry}, is one among the few that remains resistant to adversarial attacks. 

On the other hand, DNNs are often found to be highly over-parameterized. Network pruning 
%{\color{red}{(add citations here, waiting for related works \citep{})}} 
is shown to be an outstanding method which significantly reduces the model size. Typical pruning algorithms follow the three-stage 'training-pruning-fine-tuning' pipelines, where 'unimportant' weights are pruned according to certain pruning strategies, such as magnitudes of weights. 
%{\color{add citation \citep{}}}. 
However, as observed in \citet{liu2018rethinking}, fine-tuning a pruned model with inherited weights only gives comparable or worse performance than training that model with randomly initialized weights, which suggests that the inherited 'important' weights are not necessarily useful for fine-tuning. We argue below that the change of model outputs dynamics is a potential reason for this phenomenon. 

\begin{figure*}[t!]
    \centering
    %\subfigure[]{\includegraphics[scale=0.19]{images/Valiloss-cropped.pdf}}
    \subfigure[]{\includegraphics[scale=0.2]{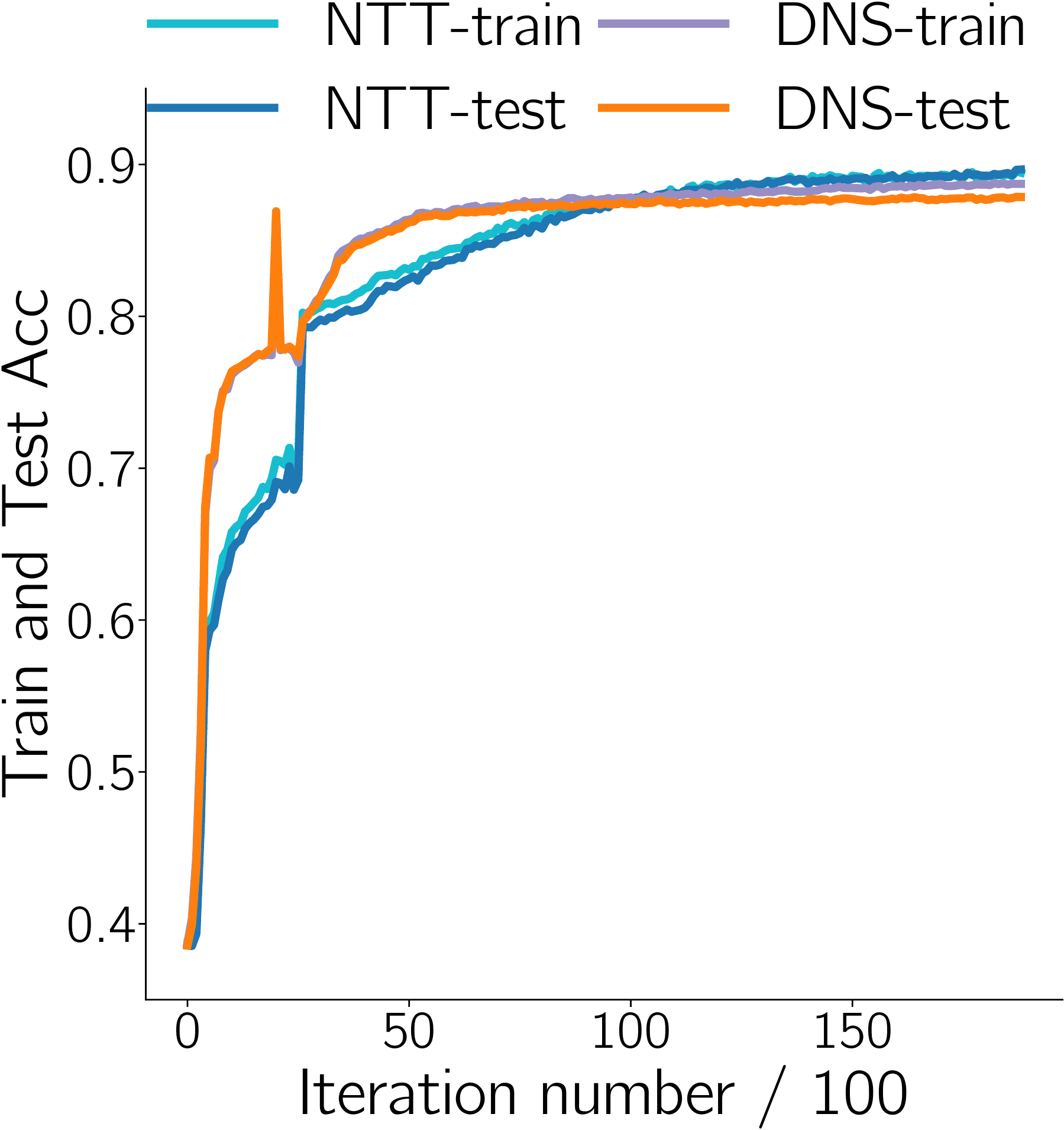}}
    \subfigure[]{\includegraphics[scale=0.2]{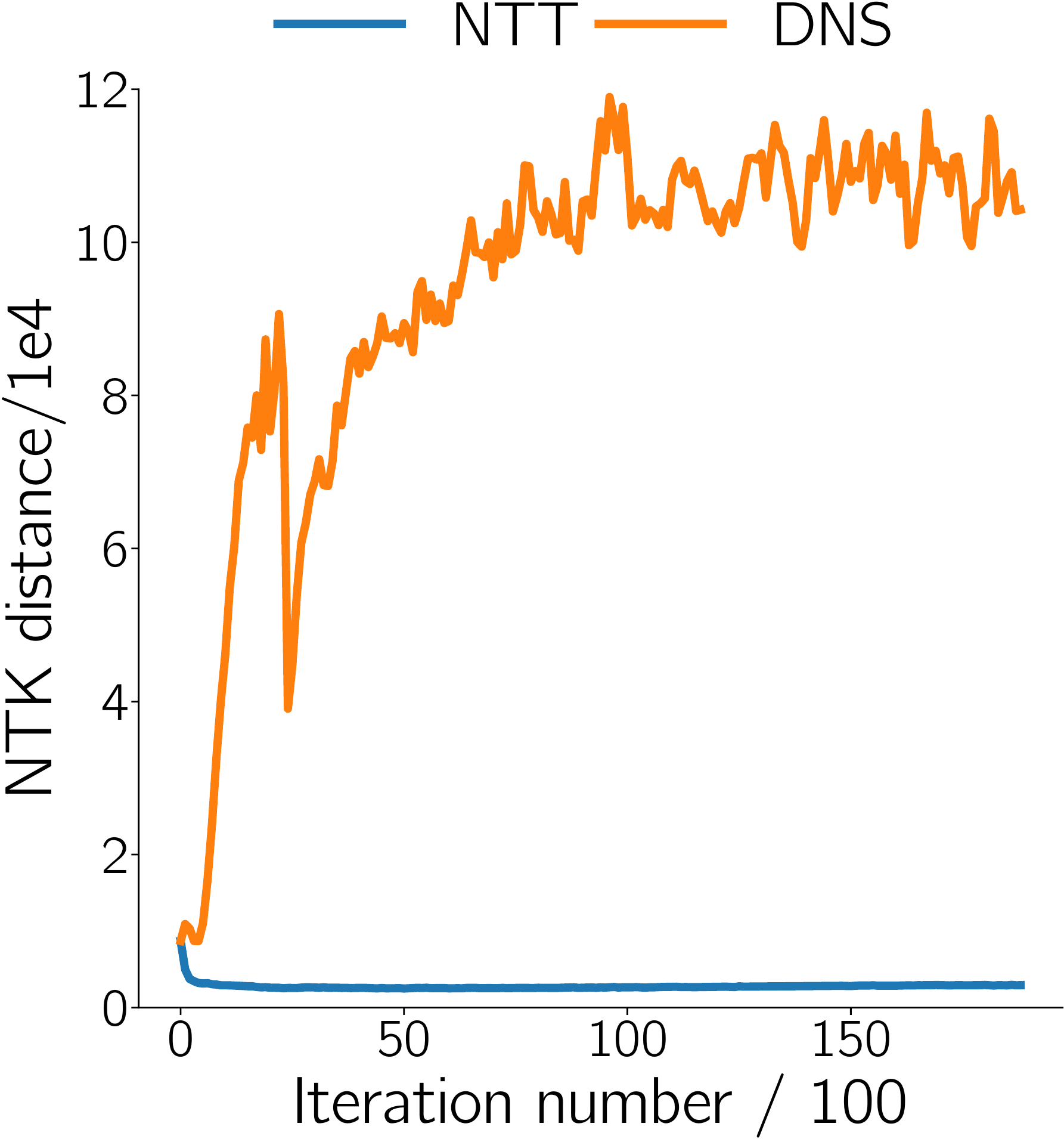}}
    \subfigure[]{\includegraphics[scale=0.2]{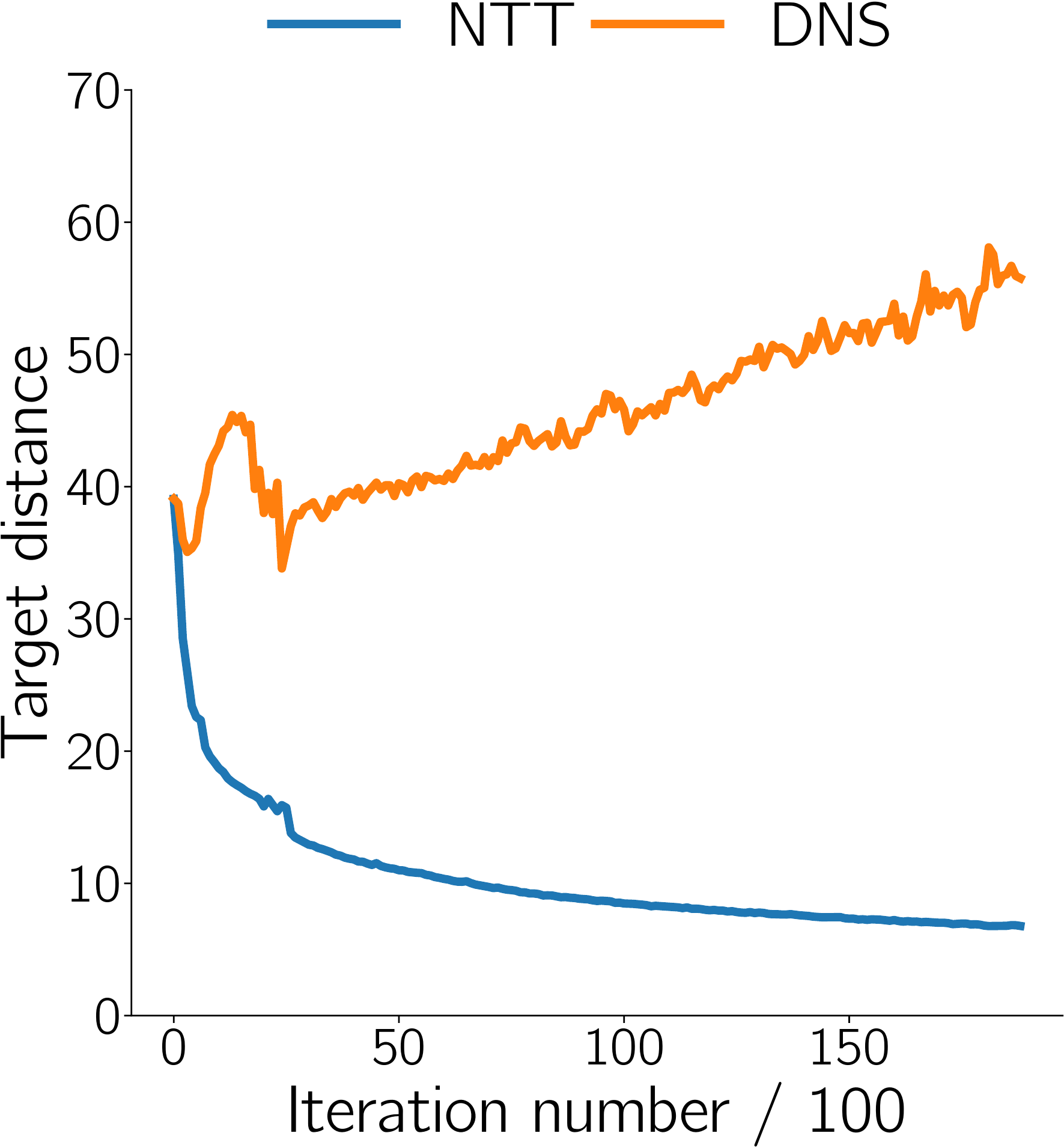}}
    \subfigure[]{\includegraphics[scale=0.2]{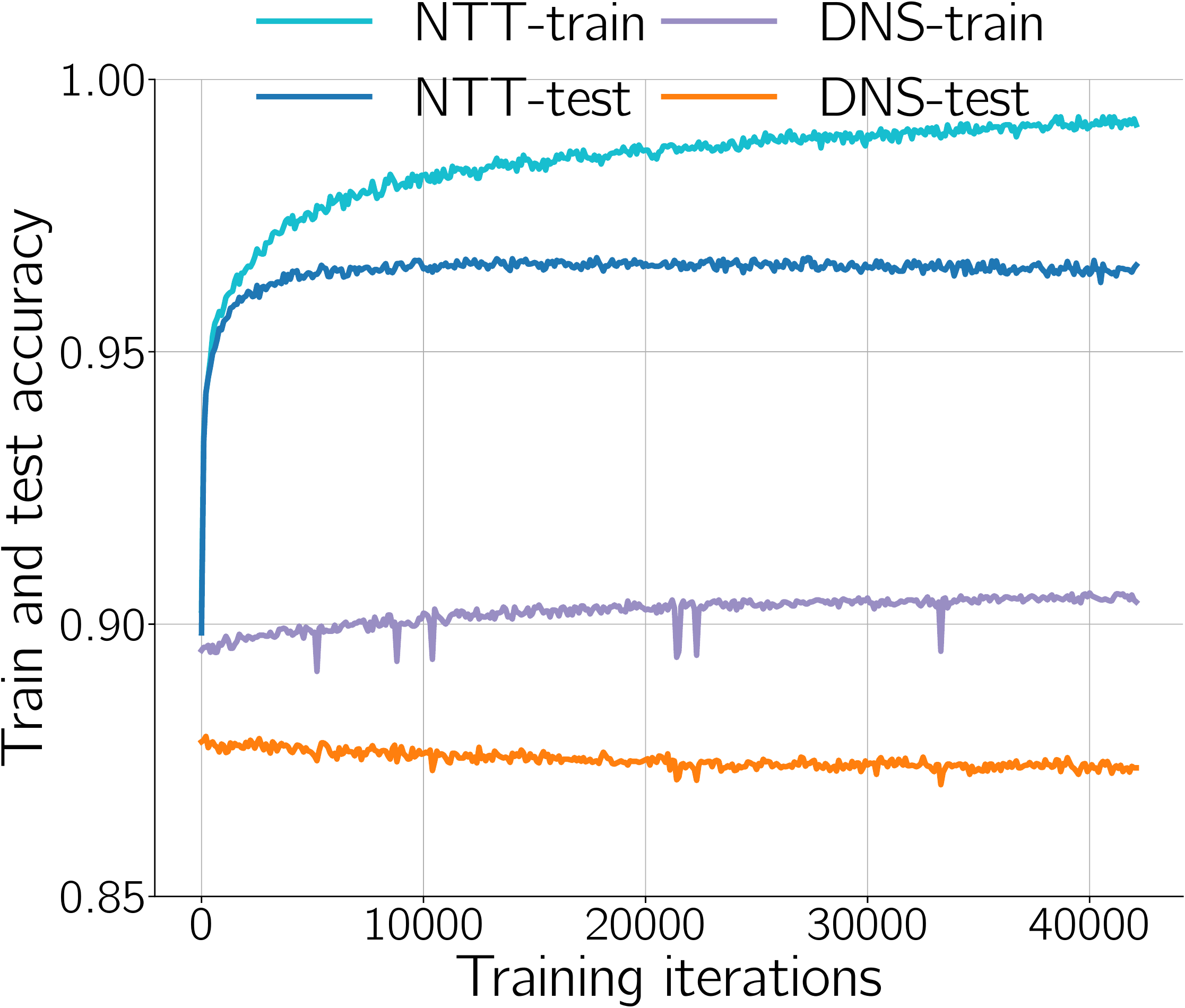}}
    \caption{(a)-(c) Statistics of NTT and DNS during mask searching. (d) Train and test accuracy during fine-tuning.}
    \label{fig:NTK}
\end{figure*}

As proposed in \citet{EvolveLinear}, the dynamics of model outputs can be completely described by the \textit{Neural Tangent Kernel} (NTK) and the initial predictions. Hence the difference of dynamics between two neural networks can be quantified by the difference of their NTKs and initial predictions. Based on this result, \citet{NTKLTH} proposed \textit{Neural Tangent Transfer} (NTT) to find trainable sparse sub-network structure which preserves the dynamics of model outputs by controlling the NTK distance and target distance between dense and sparse networks. In Figure \ref{fig:NTK}, we empirically compare various statistics of NTT with the well-known \textit{Dynamics Network Surgery} (DNS) proposed in \citet{guo2016dynamic} during mask searching and retraining/fine-tuning procedures. In Figure \ref{fig:NTK} (a), train and test accuracy increase during mask search for both NTT and DNS. This indicates that both methods successfully find sparse network with good performance. However, as shown in Figure \ref{fig:NTK} (b) and (c), the NTK distances and target distances between dense and sparse networks obtained by NTT remain in a low scale, while for DNS these two quantities blow up. This indicates that DNS flows in a different way as NTT, which lead to a different dynamics as the original dense network. As a result, we can see in Figure \ref{fig:NTK} (d) that the sparse network found by DNS is harder to be fine-tuned, while we can train the sparse network obtained by NTT from scratch to get a better performance. This observation suggests that preserving the dynamics of outputs does help to find trainable sparse structures. Experimental details will be presented in the supplementary materials.

On the other hand, \citet{LTH} conjectured the \textit{Lottery Ticket Hypothesis} (LTH), which states that there exists sub-network structure which can reach comparable performance with the original network if trained in isolation. Such sub-network is called \textit{winning ticket}. The existence of winning tickets allows us to train a sparse network from scratch with desirable performance. In particular, NTT as a foresight pruning method, provides a verification of LTH in natural training scenarios. Inspired by this observation, we consider the existence of winning ticket in adversarial context, which also preserves the dynamics of adversarial training. We call such a sparse structure an \textit{Adversarial Winning Ticket} (AWT). The benefit of looking for AWT is that its robustness is guaranteed by robustness of dense adversarial training, which has been theoretically \citep{AdvAnalysis} and empirically \citep{madry} justified. 

We briefly summarize the the contributions of this paper as follows: 
\begin{itemize}
    \item We systematically study the dynamics of adversarial training and propose a new kernel to quantify the dynamics. We refer this kernel as \textit{Mixed Tangent Kernel} (MTK).
    \item We propose a method to find AWT, which can be used to verify the LTH in adversarial context. Unlike other pruning methods in adversarial setting, AWT is obtained at initialization.  
    \item We conduct various experiments on real datasets which show that when fully trained, the AWT found by our method can achieves comparable performance when compared to dense adversarial training. These results verify the LTH empirically. 
\end{itemize}

The rest of this paper is organized as follows: In Section \ref{sec: related}, we discuss the related works. In Section \ref{sec: theory}, we develop the theory of adversarial training dynamics and state the existence theorem. In Section \ref{sec: exp} we experiment on real datasets to test the performance of AWT. Finally we conclude and discuss some possible future works in Section \ref{sec: discussion}. Proof details and additional experimental results are given in the Appendix. 

\section{Related Works}\label{sec: related}
\subsection{Adversarial Robust Learning}\label{sec:at}
The study of adversarial examples naturally splits into two areas: attack and defense. Adversarial attack methods aim to fool state-of-the-art networks. In general, attack methods consist of white-box attack and black-box attack, depending on how much information about the model we can have. White-box attacks are widely used in generating adversarial examples for training or testing model robustness where we can have all information about the model. This includes \textit{Fast Gradient Sign Method} (FGSM) \citep{ExplainAd}, \textit{Deep Fool} \citep{deepfool}, \textit{AutoAttack} \citep{autoattack} and so on. Black-box attacks \citep{zoo, SurrogateBlackAttack} are usually developed to attack model in physical world, therefore we have very limited information about the model structure or parameters. 

Meanwhile, defense methods have been studied to train an adversarial robust network which can prevent attacks from adversarial examples. Augmentation with adversarial examples generated by strong attack algorithms has been popular in the literature. \citet{madry} motivates \textit{Projected Gradient Descent} (PGD) as a universal 'first order adversary' and solve a min-max problem by iteratively generating adversarial examples and parameter updating on adversarial examples. Such method is referred as \textit{adversarial training} (AT). The convergence and performance of AT have been theoretically justified by recent works \citep{AdvAnalysis, zhang2020over, gao2019convergence}. Also, methods \citep{FreeAdv, FastAdv} have been developed to speed up the training of AT for large scale datasets such as ImageNet. 

\subsection{Sparse Learning}\label{sec:adv-pruning}
\paragraph{Pruning Methods} Network pruning \citep{han2015learning, guo2016dynamic,zeng2018mlprune,li2016pruning,luo2017thinet,he2017channel,zhu2017prune,zhou2021efficient,zhou2021effective} has been extensively studied in recent years for reducing  model size and improve the inference efficiency of deep neural networks. Since it is a widely-recognized property that modern neural networks are always over-parameterized, pruning methods are developed to remove unimportant parameters in the fully trained dense networks to alleviate such redundancy. According to the  granularity of pruning, existing pruning methods can be roughly divided into two categories, i.e., unstructured pruning and structured pruning. The former one is also called weight pruning, which removes the unimportant parameters in an unstructured way, that is, any element in the weight tensor could be removed. The latter one removes all the weights in a certain group together, such as kernel and filter. Since structure is taken into account in pruning, the pruned networks obtained by structured pruning are available for efficient inference on standard computation devices. In both structured and unstructured pruning methods, their key idea is to propose a proper criterion (e.g., magnitude of the weight) to evaluate the importance of the weight, kernel or filter and then remove the unimportant ones. he results in the literature \citep{guo2016dynamic,liu2018rethinking,zeng2018mlprune,li2016pruning} demonstrate that pruning methods can significantly improve the inference efficiency of DNNs with minimal performance degradation, making the deployment of modern neural networks on resource limited devices possible.

Along the research line of LTH, recent works, e.g., SNIP \citep{snip} and GraSP \citep{grasp},
empirically show that it is possible to find a winning ticket at intialization step, without iteratively training and pruning procedure as the classical pruning methods. The key idea is to find a sub-network, which preserves the gradient flow at initialization. NTT \citep{NTKLTH} utilizes the NTK theory and prune the weights by preserving the training dynamics of model outputs, which is captured by a system of differential equations.  

\paragraph{Adversarial Pruning Methods} Recent works by \citet{guo2018sparse} have proven sparsity can improve adversarial robustness. A typical way of verifying the \textit{Lottery Ticket Hypothesis} (LTH) is finding the winning ticket by iteratively training and pruning. Such strategy is also considered in adversarial context \citep{AdvLTH, wang2020achieving, li2020towards, gilles2020lottery}, with natural training replaced by adversarial training. Other score based pruning methods have also been considered \citep{sehwag2020hydra}. Recent work \citep{fu2021drawing} also considered sub-network structure with inborn robustness without training. 

Other works bring in the model compression methods into sparse adversarial training. \citet{gui2019model} integrates pruning, low-rank factorization and quantization into a unified flexible structural constraint. \citet{AdvModelCompression} proposes concurrent weight pruning to reach robustness. Both works introduce certain sparse constraints and solve the optimization problem under \textit{alternating direction method of multipliers} (ADMM) framework. 

\subsection{Neural Tangent Kernel}\label{sec:ntk}
Recent works by~\citet{ntk} consider the training dynamics of deep neural network outputs and proposed the \textit{Neural Tangent Kernel} (NTK). \citet{ntk} shows under the infinite width assumption, NTK converges to a deterministic limiting kernel. Hence the training is stable under NTK. NTK theory has been widely used in analyzing the behavior of neural networks. \citet{EvolveLinear} proves infinitely wide \textit{multilayer perceptrons} (MLP) evolve as linear model, which can be described as the solution of a different equation determined by the NTK at initialization. \citet{arora2019exact} further shows that ultra-wide MLPs behave as kernel regression model under NTK. These results have also been applied to different areas in deep learning, such as foresight network pruning \citep{NTKLTH}, federated learning~\citep{huang2021fl} and so on.

\section{Dynamics Preserving Sub-Networks}\label{sec: theory}
In this section, we verify the \textit{Lottery Ticket Hypothesis} (LTH) in adversarial context by showing the existence of \textit{Adversarial Winning Ticket} (AWT). We first derive the equations describing the dynamics of adversarial training. Then we propose the optimization problem of finding the AWT by controlling the sparse adversarial training dynamics. Finally we prove an error bound between the dense model outputs and the sparse model outputs, which implies the AWT has the desired theoretical property. 

\subsection{Dynamics of Adversarial Training}\label{sec: dyn-theo}
Let $\mathcal{D}=X\times Y=\{(x_1,y_1),\cdots, (x_N,y_N)\}$ be the empirical data distribution, $f_\theta(x)\in\mathbb{R}^{k\times 1}$ the network function defined by a fully-connected network\footnote{We make this assumption because the NTK theory we are going to apply is valid for fully-connected networks only.},  and $f_\theta(X)=\mathrm{vec}\big([f_\theta(x)]_{x\in X}\big)\in\mathbb{R}^{k|\mathcal{D}|\times 1}$ be the model outputs on training data. 

Recall that adversarial training solves the following optimization problem:
\begin{equation}\label{eq: adv op}
\begin{split}
   \min_\theta \mathcal{L}&=\underset{{(x,y)\sim\mathcal{D}}}{\mathbb{E}}\max_{r\in S_\varepsilon(x)}\ell(f_\theta(x+r), y)\\
   &= \frac{1}{N}\sum_{i=1}^N \max_{r_i\in S_\varepsilon(x_i)}\ell(f_\theta(x_i+r_i), y_i)  
\end{split}
\end{equation}
The inner sub-problem of this min-max optimization problem is usually solved by an effective attack algorithm. If we use $\tilde{x}_j$ denote the adversarial example of $x$ obtained at $j$-th step, then any $k$ steps $\ell_p$ ($1\le p\le\infty$) iterative attack algorithm with allowed perturbation strength $\varepsilon$ can be formulated as follows:
\begin{equation}\label{eq:attack-algorithm}
    \begin{split}
        &\tilde{x}_0=x, \quad
        \tilde{x}_{t}=\tilde{x}_{t-1}+r_t\quad \tilde{x}=\tilde{x}_k\\
        &\mathrm{s.t.} \norm{r_i}_p\le\delta\quad\norm{\sum r_i}_p\le\varepsilon\quad\forall 1\le t\le k
    \end{split}
\end{equation}
In practice, PGD attack as proposed in \citet{madry} is a common choice. Also, for bounded domains, clip operation need to be considered so that each $\tilde{x}_t$ still belongs to the domain. However, such restriction is impossible to be analyzed in general. So we remove the restriction by assuming the sample space is unbounded. In this case, adversarial training algorithm updates the parameters by stochastic gradient descent on adversarial example batches. To be precise, we have the following discrete parameter updates:
\begin{equation}
    \theta_{t+1}=\theta_t-\eta\frac{\dd \mathcal{L}}{\dd \theta}(\tilde{X}_t)
\end{equation}
For an infinitesimal time $\dd t$ with learning rate $\eta_t=\eta \dd t$, one can obtain the continuous gradient descent by chain rules as follows:
\begin{equation}\label{eq: sgd-adv-cont}
    \begin{split}
    \frac{\dd \theta_t}{\dd t}=\frac{\theta_{t+\dd t}-\theta_t}{\dd t}
    =-\eta\nabla^T_\theta f_t(\tilde{X}_t)\nabla_{f_t}\mathcal{L}(\tilde{X}_t)
    \end{split}
\end{equation}
where we use the short notation $f_t(x)=f_{\theta_t}(x)$ and the following notation for convenience\footnote{We drop the labels $Y$ here since in adversarial training, the labels assigned to adversarial examples are the same as the clean ones.}:
\begin{equation}
    \nabla_f\mathcal{L}(X)=\begin{bmatrix}
    |\\ \nabla_f\ell(f(x_i), y_i)\\|
    \end{bmatrix}
\end{equation}
Accordingly, we can obtain the following theorem relating to dynamics of adversarial training. 
\begin{thm}\label{thm:dyn}
Let $f_t(x)$ be the timely dependent network function describing adversarial training process and $\tilde{X}_t$ the adversarial examples generated at time $t$ by any chosen attack algorithm. Then $f_t$ satisfies the following differential equation:
\begin{equation}\label{eq:adv-dyn}  
\begin{split}
    \frac{\dd f_t}{\dd t}(X)&=\nabla_\theta f_t(X)\frac{\dd\theta_t}{\dd t}\\
    &=-\frac{\eta}{N}\nabla_\theta f_t(X)\nabla_\theta^T f_t(\tilde{X}_t)\nabla_{f_t}\mathcal{L}(\tilde{X}_t)
\end{split}
\end{equation}
\end{thm}
Equation \eqref{eq:adv-dyn} is referred as the dynamics of adversarial training because it describes how the adversarially trained network function $f_t$ evolves along time $t$. On the other hand, training a  adversarial robust network $f_\theta$ is the same as solving Equation \eqref{eq:adv-dyn} for given certain initial conditions.

Let $\Theta_t(X,Y)=\nabla_\theta f_t(X)\nabla_\theta^Tf_t(Y)$, then $\Theta(X,X)$ is the well-known empirical \textit{Neural Tangent Kernel} (NTK) %introduced in \cite{ntk}
, which describes the dynamics of natural training as studied in \cite{EvolveLinear}. To be precise, if $f^{nat}$ is the model function of natural training, $\theta^{nat}$ the corresponding parameters and $\mathcal{L}_{nat}$ the corresponding loss, then the dynamics of natural training are given by
\begin{align}
        &\frac{\dd \theta^{nat}_t}{\dd t}=-\eta\nabla_\theta^T f^{nat}_t(X)\nabla_{f^{nat}_t}\mathcal{L}_{nat}(X)\label{eq:sgd-nat}\\
        &\frac{\dd f^{nat}_t}{\dd t}(X)=-\eta\Theta_t(X,X)\nabla_{f^{nat}_t}\mathcal{L}_{nat}(X)\label{eq:dyn-nat}
\end{align}
Detailed calculations can be found in \cite{EvolveLinear}. If we compare Equation \eqref{eq: sgd-adv-cont} with Equation \eqref{eq:sgd-nat}, we see that the gradient descent of adversarial training can be viewed as natural training with clean images $X$ replaced by adversarial images $\tilde{X}_t$ at each step. This matches our intuition because in practice, the parameter update is based on the adversarial examples as we discussed above, so adversarial training is closely related to natural training on adversarial images. 

However, if we compare Equation \eqref{eq:adv-dyn} and Equation \eqref{eq:dyn-nat}, we can see from the evolution of the model outputs that the usual NTK is now replaced by $\Theta_t(X,\tilde{X}_t)=\nabla_\theta f_t(X)\nabla_\theta^T f_t(\tilde{X}_t)$, which we call \textit{Mixed Tangent Kernel} (MTK). Unlike NTK, MTK is not symmetric in general. It involves both clean images $X$ and adversarial images $\tilde{X}_t$. This indicates adversarial training is not simply a naturally model training on adversarial images, but some more complicated training method continuously couples clean images and adversarial images during training procedure. This coupling of clean images and adversarial images gives an intuition why adversarial training can achieve both good model accuracy and adversarial robustness. 

\subsection{Finding Adversarial Winning Ticket}\label{sec: adv-win-ticket}
To verify the LTH in adversarial setting, we need to find out a trainable sparse sub-network which has similar training dynamics as the dense network. We are then aiming to find a mask $m$ with given sparsity density $p$ such that the sparse classifier $f^s(x)=f_{m\odot\theta}(x)$ can be trained to be adversarial robust from scratch\footnote{Without loss of generality, we use superscript $s$ to mean items correspond to sparse networks, while items without superscript correspond to dense networks.}. For simplicity, we assume, as in \citet{EvolveLinear} and \citet{NTKLTH}, the cost function to be squared loss\footnote{Norms without subscript will denote $\ell_2$ norm.} $\ell(f_\theta(x),y)=\displaystyle\frac{1}{2}\norm{f_\theta(x)-y}^2$. A discussion of other loss functions, such as cross-entropy, is given in Appendix \ref{sec:extensions}. Let $\tilde{X}$ be the collection of adversarial examples as above, then 
\begin{equation}
    \nabla_{f_t}\mathcal{L}(\tilde{X}_t)=f_t(\tilde{X}_t)-Y
\end{equation}
And therefore, the dynamics of model outputs of dense network in Equation \eqref{eq:adv-dyn} becomes
\begin{equation}\label{eq:adv-dyn-l2}  
     \frac{\dd f_t}{\dd t}(X)=-\frac{\eta}{N}\Theta_t(X,\tilde{X}_t)(f_t(\tilde{X}_t)-Y)
\end{equation} 

\begin{algorithm}[t!]\caption{Finding Adversarial Winning Ticket}\label{alg:awt}
	\begin{algorithmic}[1]
		\STATE {\bfseries Input:} clean images $X$, labels $Y$, model structure $f$, dense initialization $\theta_0$, learning rate $\eta$, adversarial perturbation strength $\varepsilon$, sparsity level $k$, mask update frequency $T_m$.
		\STATE {\bfseries Initialize:} initial weight $w_0=\theta_0$, initial binary mask $m$ based on $w_0$, adversarial images $R_0$, $t=1$.
		\FOR{$t=1$ to $N$}
		\STATE Sample a mini batch $S$ and calculate the gradient $\nabla_{w}\mathcal{L}_{awt}$ on $S$. 
		\STATE $w\leftarrow w-\eta\nabla_{w}\mathcal{L}_{awt}-\beta m\odot w$
		\IF{$t\ \%\ T_m = 0$}
		\STATE update $m$ according to magnitudes of current $w$
		\ENDIF
		\ENDFOR
		\STATE {\bfseries Return:}  $m$. 
	\end{algorithmic}\label{alg:adv-winning-ticket}
\end{algorithm}

To achieve our goal, note that the dense classifier $f_t(x)$ in Theorem \ref{thm:dyn} converges eventually to the solution of adversarial training, so it is supposed to be adversarial robust if fully-trained. On the other hand, the dynamics of the sparse adversarial training $f^s_t(x)$ can be described similarly as:
\begin{equation}\label{eq:adv-dyn-l2-sp}
    \frac{\dd f^s_t}{\dd t}(X)=-\frac{\eta}{N}\Theta^s_t(X,\tilde{X}^s_t)(f^s_t(\tilde{X}^s_t)-Y)
\end{equation}
where $\tilde{X}^s_t$ is the collection of adversarial images corresponding to sparse network and $\Theta^s_t(X,\tilde{X}^s_t)$ is the MTK of sparse classifier. Therefore, to get the desired mask $m$, it is sufficient to make $f^s_t(X)\approx f_t(X)$ for all $t$. According to Equation \eqref{eq:adv-dyn-l2} and Equation \eqref{eq:adv-dyn-l2-sp}, this can be achieved by making the MTK distance and adversarial target distance between dense and sparse networks close enough at any time $t$ in the training. That is to say,
\begin{equation}
    \Theta_t(X,\tilde{X}_t)\approx \Theta^s_t(X,\tilde{X}^s_t)\quad\quad f_t(\tilde{X}_t)\approx f^s_t(\tilde{X}^s_t)
\end{equation}
for all $t$. Under mild assumptions, we may expect all these items are determined at $t=0$ because of the continuous dependence of the solution of differential equations on the initial values. This then leads to the consideration of the following optimization problem:
\begin{equation}\label{eq:awt-op}
\begin{split}
    \min_{m}\mathcal{L}_{awt}=&
        \frac{1}{N}\norm{f_{\theta_0}(\tilde{X}_0)-f^s_{m\odot\theta_0}(\tilde{X}^s_0)}^2\\
        &+\frac{\gamma^2}{N^2}\norm{\Theta_0(X, \tilde{X}_0)-\Theta^s_0(X, \tilde{X}^s_0)}_F^2
\end{split}
\end{equation}
where $\norm{\cdot}$ is the $\ell_2$ norm of vectors and $\norm{\cdot}_F$ is the Frobenius norm of matrices. In equation \eqref{eq:awt-op}, the first and second items in the right hand side are referred as {\it target distance} and {\it kernel distance}, respectively.  We call the resulting mask \textit{Adversarial Winning Ticket} (AWT). Our method is summarized in algorithm \ref{alg:awt}. Since the binary mask $m$ cannot be optimized directly, instead we train a student network $f_{m\odot w}(x)$. The mask $m$ is then updated according to the magnitudes of current weights $w$ after several steps, which is specified by the mask update frequency. To get sparse adversarial robust network, the obtained AWT $f_{m\odot\theta_0}(x)$ will be adversarially trained from scratch.

\begin{figure}[t!]
    \centering
    \includegraphics[scale=1, trim = 9.2cm 7.06cm 6cm 8.8cm, clip]{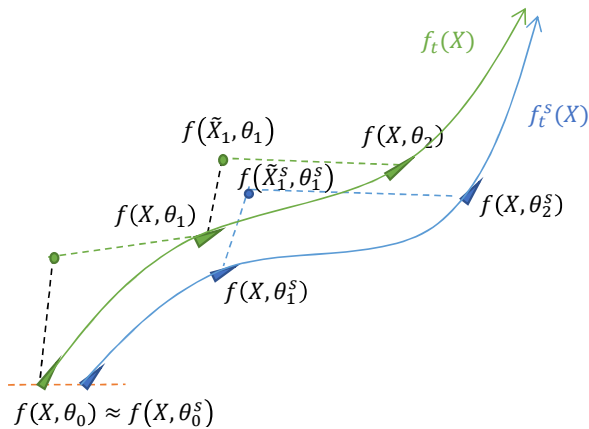}
    \caption{Schematic illustration of networks' outputs evolution under adversarial training. Solid lines represent continuous training dynamics of dense (green) and sparse (blue) networks. Triangular marks locate model outputs at each step under gradient descent. Vertical dash lines represent adversarial attacks and correspondingly, horizontal dash lines represent parameter updates with respect to given adversarial examples. The SGD process from $t=1$ to $t=2$ is marked.
    %, for example $(X,\theta_1)$ means $f_{\theta_1}(X)$. e.g. $(X,\theta_1)\to(\tilde{X}_1,\theta_1)$ means getting adversarial example $\tilde{X}_1$ from $X$ with respect to $\theta_1$.  Both sparse and dense networks start at the same initialization $\theta_0=\tilde{\theta}_0$ (but with different model outputs in general). Optimization problem \eqref{eq:awt-op} suggests the two curves $f_t(x)$ and $f^s_t(x)$ are close if we can control the distance of start positions and tangent directions at start point $t=0$.
    }
    \label{fig:awt}
\end{figure}

This intuition can be further illustrated by Figure \ref{fig:awt}. For each iteration of gradient descent ($t=1$ to $t=2$ in the figure), adversarial training contains two steps: adversarial attack (vertical dash line) and parameter update (horizontal dash line). Our goal is to make the blue curve (sparse) close to the green one (dense). This can be done by making the attack and parameter update curves of sparse and dense networks close enough for each time $t$. However, as we can see from the figure, $f_t$ and $f^s_t$ are determined by the initial condition, hence we get the above optimization problem.  

Formally we have the following theorem to estimate the error bound between dense and sparse outputs.

\begin{thm}\label{thm:main}
Let $f_\theta(x)$ denote the dense network function. Suppose $f_\theta$ has identical number of neurons for each layer, i.e. $n_1=n_2=\cdots=n_L=n$ and assume $n$ is large enough. Denote $f^s_{m\odot\theta}(x)$ the corresponding sparse network with $1-p$ weights being pruned. Assume $f$ and $f^s$ have bounded first and second order derivatives with respect to $x$, i.e.
\begin{align*}
    &\max_{t, x}\big\{\norm{\partial_{x} f_t}_q, \norm{\partial_{x} f^s_t}_q\big\}\le C_{1,q}\\
    &\max_{t, x}\big\{\norm{\partial_{xx}^2f_t}_{p,q}, \norm{\partial_{xx}^2f^s_t}_{p,q}\big\}\le C_{2,q}
\end{align*}
where we choose an $\ell_p$ attack to generate adversarial examples such that $q$ is the conjugate of $p$ in the sense of $1/p+1/q=1$.\footnote{If $p=\infty$, we take $q=1$.} Denote the optimal loss value for AWT optimization problem \eqref{eq:awt-op} to be $\mathcal{L}_{awt}^*=\alpha^2$. Then for all $t\le T$ with $T$ the stop time, with learning rate $\eta=O(T^{-1})$, we have
\begin{equation}\label{eq:main}
    \underset{{x\in\mathcal{D}}}{\mathbb{E}}\norm{f_t(x)-f^s_t(x)}^2\le 4(\alpha+4C_q\varepsilon)^2
\end{equation}
where $C_q=C_{1,q}+\varepsilon C_{2,q}$ is a constant. 
\end{thm}
Note that we put no restriction on any specific attack algorithm, hence we can choose any strong attack algorithm for generating adversarial examples. In practice, PGD attack is commonly chosen for adversarial training. Also, our theoretical results consider any $\ell_p$ attack with $1\le p\le\infty$. The uniform bound assumption of derivatives are reasonable. If we take the Taylor expansion of $f_t$ with respect to $f_0$, then the derivatives are functions of $\theta_t$. Since we apply weight decay in our training, $\theta_t$ is uniformly bounded for all $t\le T$, also we only have finite training data, so the derivatives can be assumed to be uniformly bounded. Moreover, $C_q$ can be adjusted by carefully choosing the regularizing constant of weight decay. Proof details and a discussion of the constants are presented in Appendix \ref{sec:supp-theory}.

Equation \eqref{eq:main} shows that the expected error between sparse and dense outputs are bounded by the optimal loss value and adversarial perturbation strength. In practice, the optimial loss value $\alpha^2$ and adversarial perturbation strength $\varepsilon$ are small, we may expect the output of AWT is close to dense output, hence is adversarial robust if fully-trained. Therfore Theorem \ref{thm:main} can be viewed as theoretical justification of the existence of LTH in adversarial setting, and we can find AWT by solving the optimization problem \eqref{eq:awt-op}. 

Theorem \ref{thm:main} reduces to natural training if we take $\varepsilon=0$. In this case, the AWT found is winning ticket for natural training. Hence Theorem \ref{thm:main} also verifies the classical LTH as a special case. Meanwhile, our method reduces to \textit{Neural Tangent Transfer} (NTT) in \citet{NTKLTH} and Equation \eqref{eq:main} gives an error bound of NTT. Furthermore, for ideal case when $\alpha=0$, Equation \eqref{eq:main} implies $f_t(x)=f_t^s(x)$ for all $x$, hence the dense and sparse networks have identical outputs for all time $t$, which extends Proposition 1 in \citet{NTKLTH}.

\section{Experiments}\label{sec: exp}
We now empirically verify the performance of our method.
To be precise, we first show the effectiveness of our method  in preserving the dynamics of adversarial training, that is, our method  can find a sparse sub-network, whose training dynamics are close to the dense network. 
Then we evaluate the robustness of the sparse neural networks obtained by our method. At last, we give a preliminary experimental result to show the possibility to extend our method to large-scaled problems.  

\begin{figure*}[t!]
    \centering
    \subfigure[$\ell_2, \varepsilon=2$]{\includegraphics[scale=0.28]{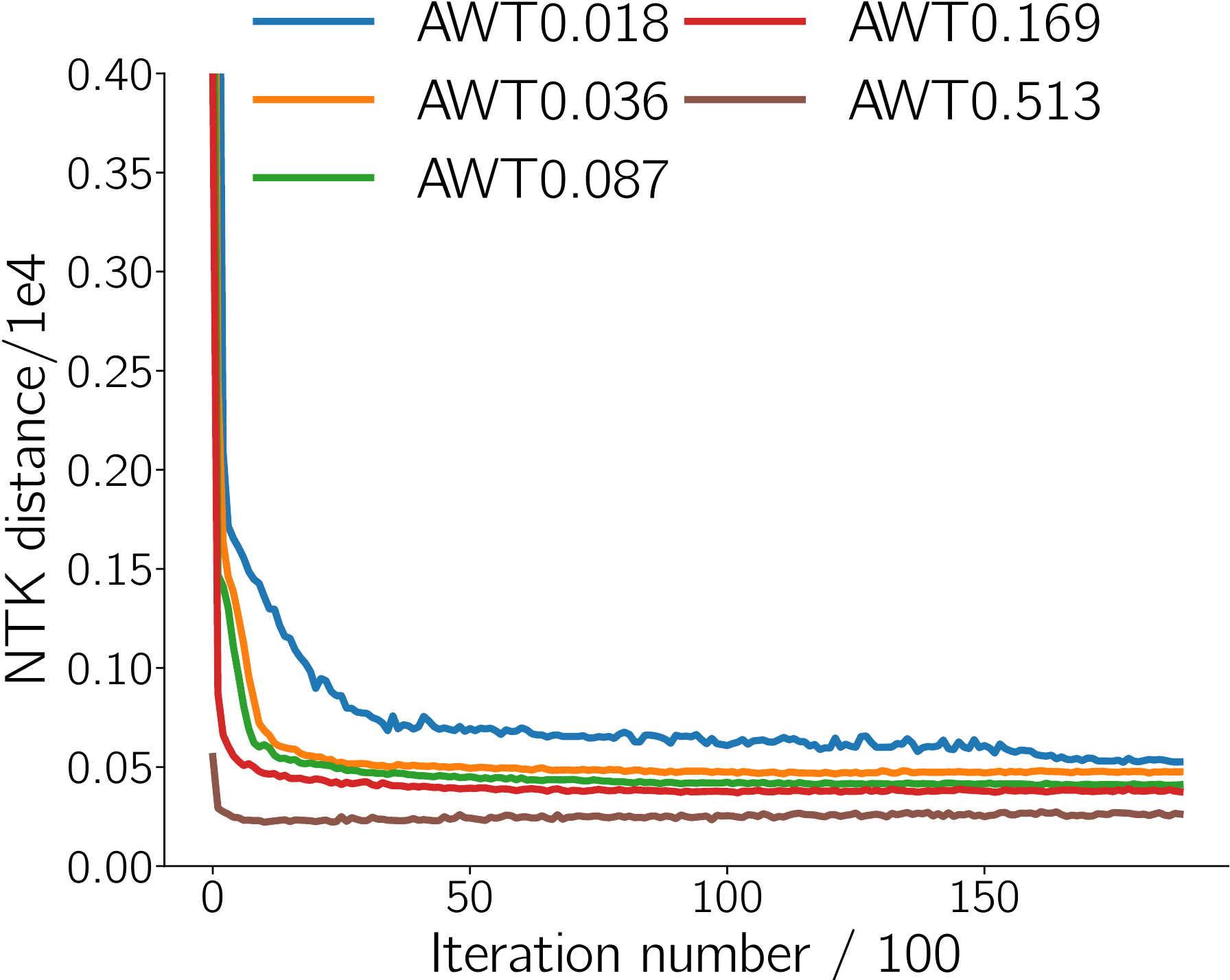}}
    \subfigure[$\ell_2, \varepsilon=2$]{\includegraphics[scale=0.28]{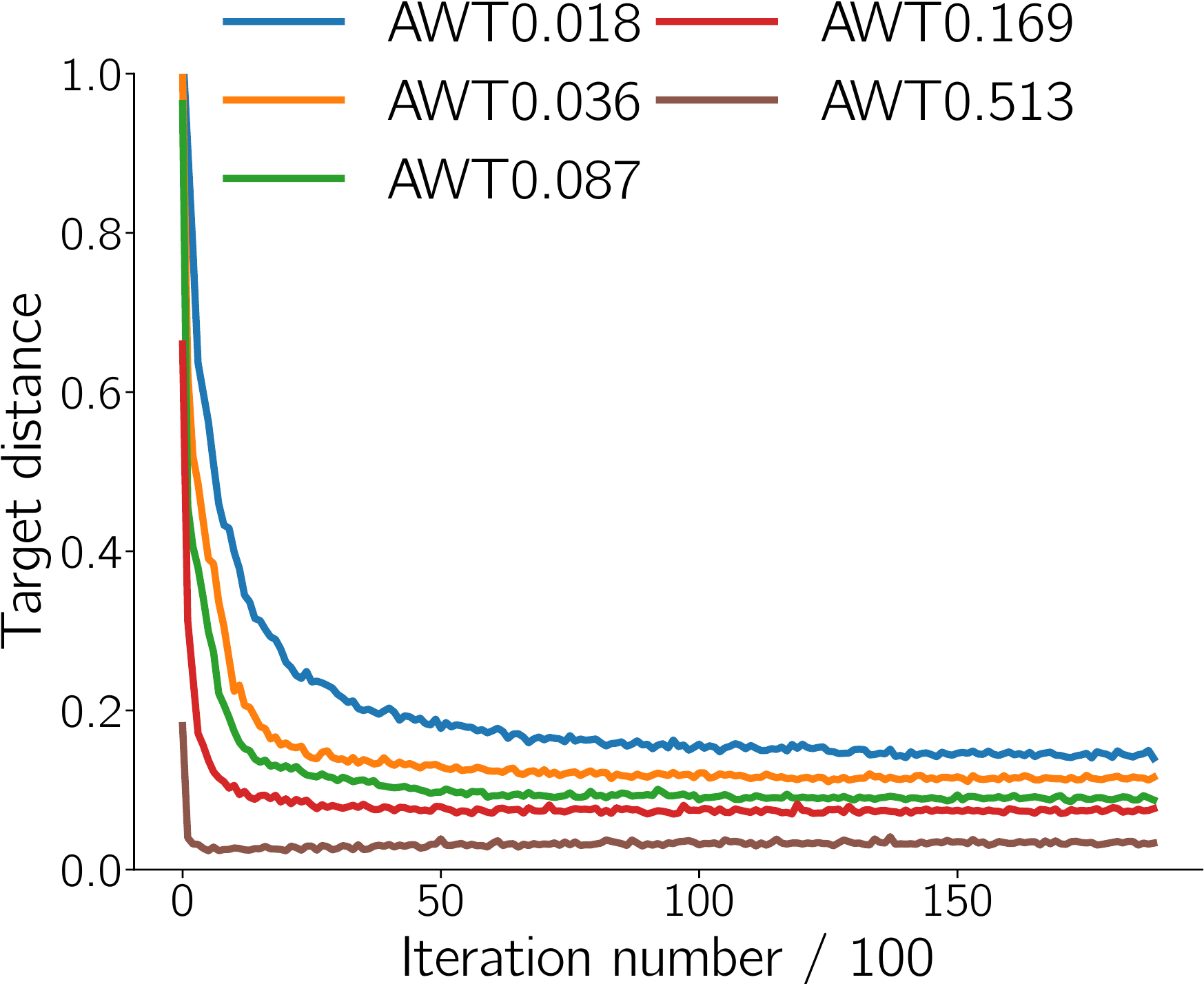}}
     \subfigure[$\ell_2, \varepsilon=2$]{\includegraphics[scale=0.28]{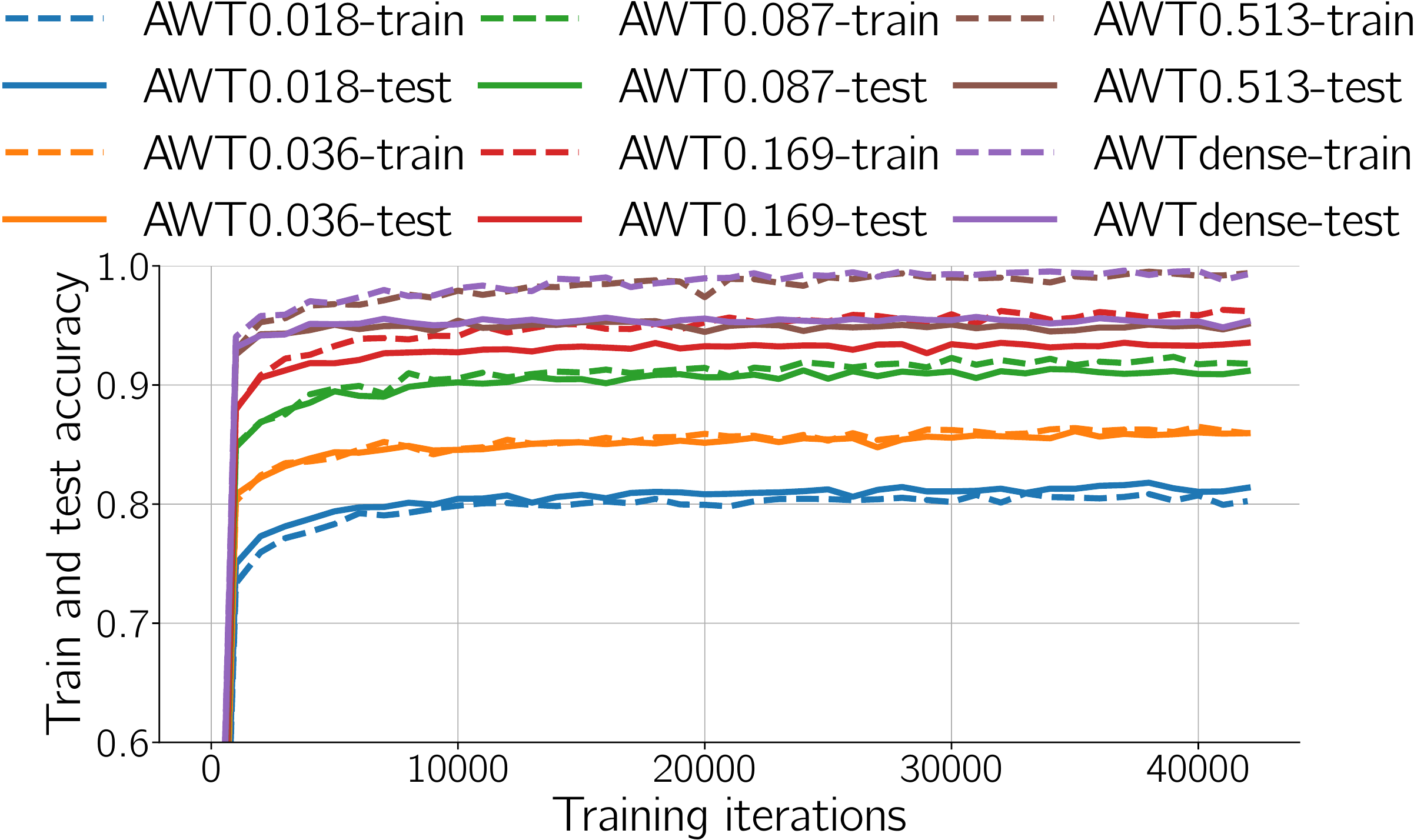}}
    \subfigure[$\ell_\infty, \varepsilon=0.3$]{\includegraphics[scale=0.28]{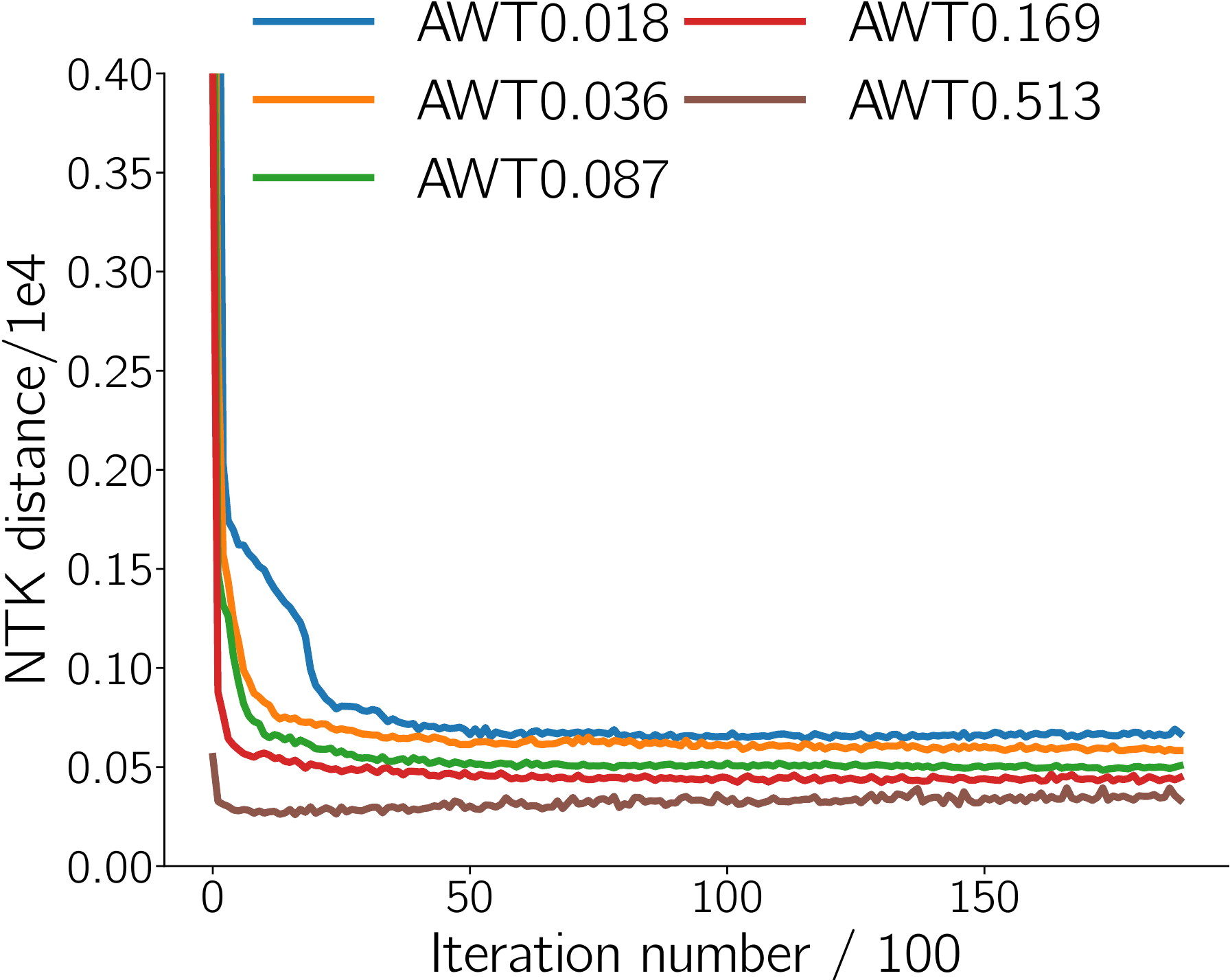}}
    \subfigure[$\ell_\infty, \varepsilon=0.3$]{\includegraphics[scale=0.28]{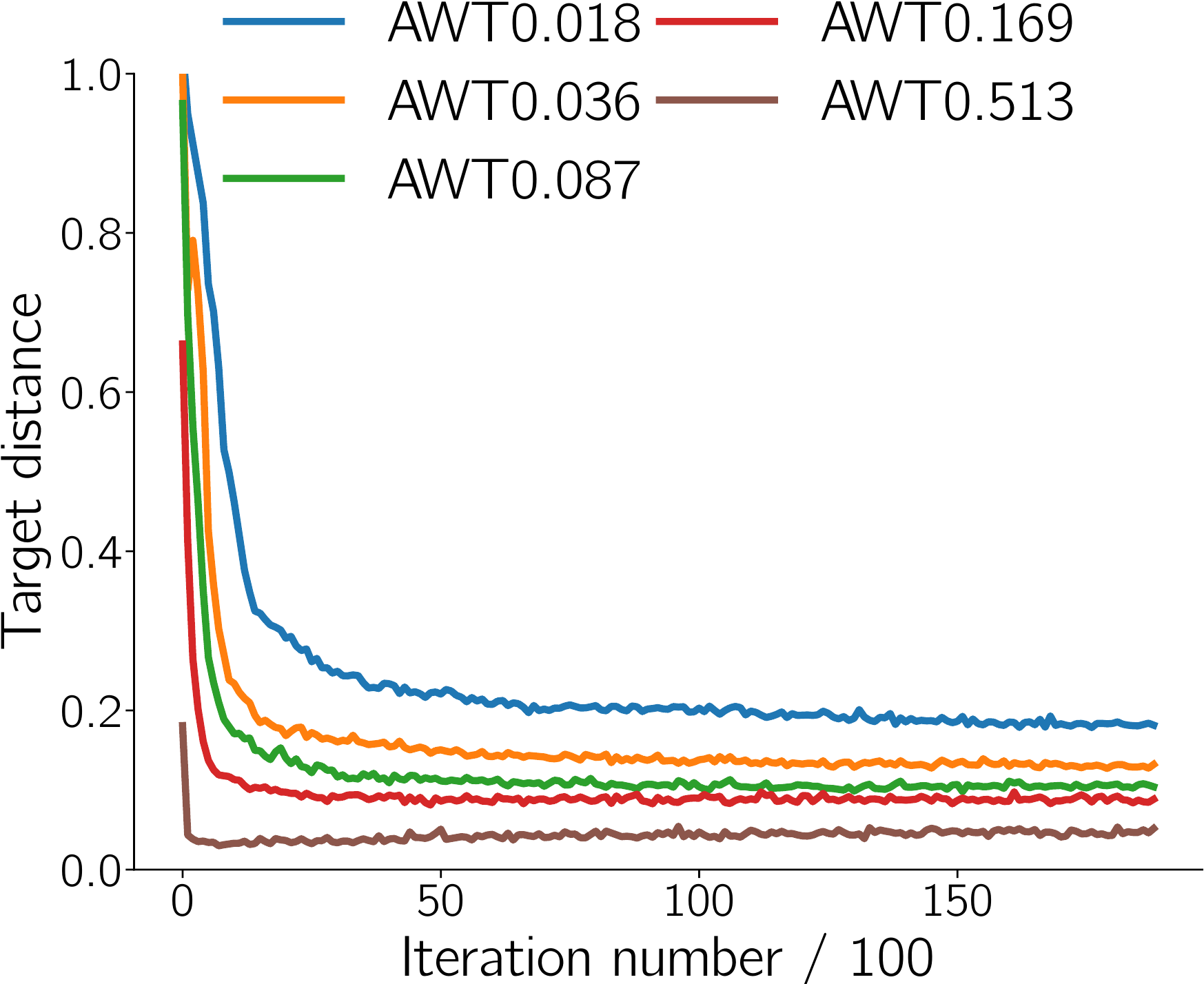}}
     \subfigure[$\ell_\infty, \varepsilon=0.3$]{\includegraphics[scale=0.28]{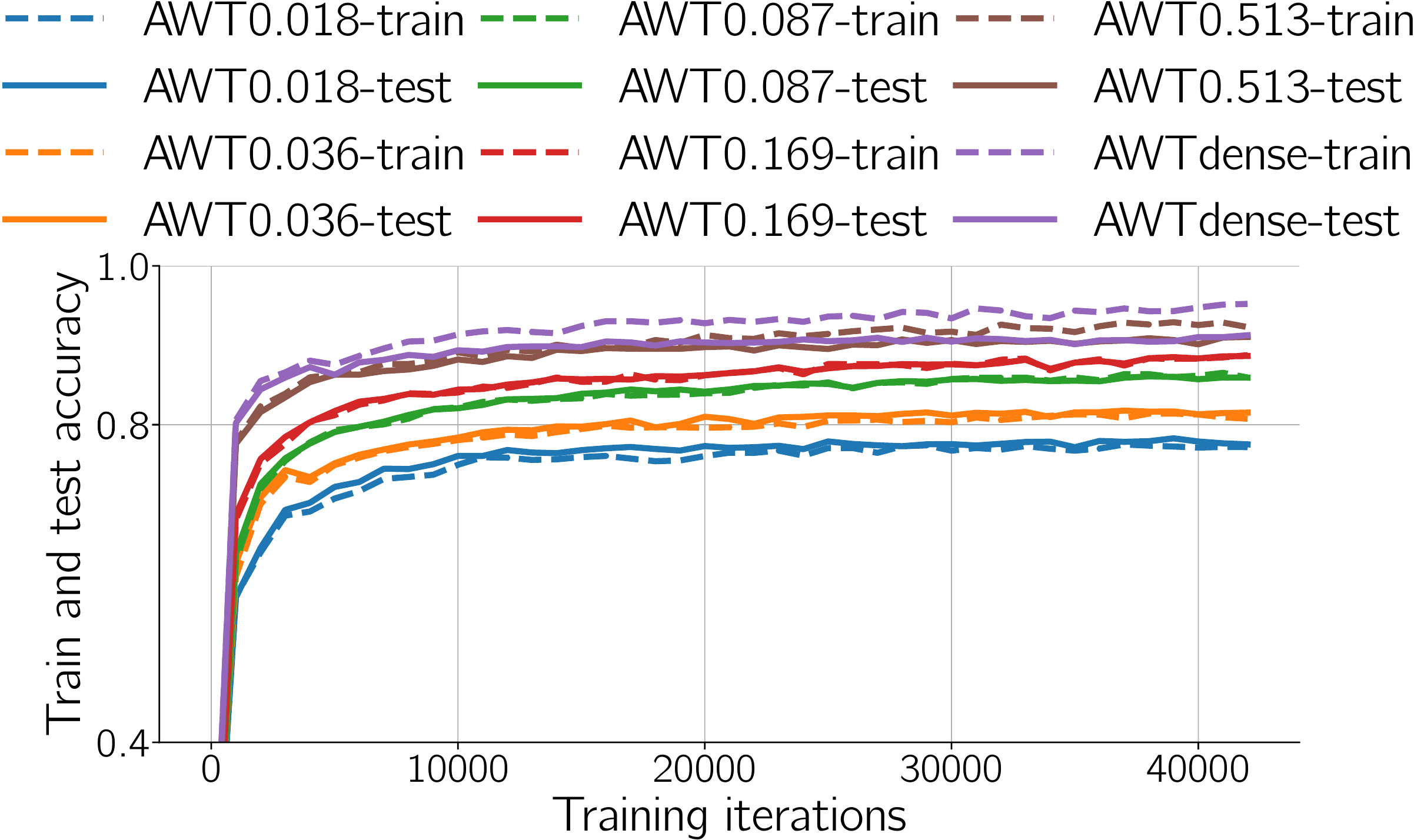}}
     %\vspace*{-5pt}
    \caption{(a) and (b) are the statistics of AWT during mask searching on MNIST with MLP over different density levels under $\ell_2$ attack. (c) presents the adversarial training and test accuracy curves in the training process under $\ell_2$ attack. (d)-(e) are the results accordingly under $\ell_\infty$ attack.}
    \label{fig:kernel-target-distance-AWT-mlp-mnist}
\end{figure*}

\subsection{Implementation}\label{sec:setting}
We conduct experiments on standard datasets, including MNIST \citep{lecun1998gradient} and CIFAR-10 \citep{krizhevsky2009learning}.
All experiments are performed in JAX \citep{jax2018github}, together with the neural-tangent library \citep{neuraltangents2020}. Due to the high computational and storage costs of NTK, following the experimental setting in \cite{NTKLTH}, we mainly evaluate our proposed method on two networks: MLP and 6-layer CNN.  The preliminary experiment of scalability in Section \ref{sec:scale} is conducted on VGG-16 with CIFAR-10.

\iffalse 
The MLP is comprised of two hidden layers containing 300 and 100 units with rectified linear unit(ReLUs) followed with a linear output layer. The CNN has two convolutional layers with 32 and 64 channels, respectively. Each convoluational layer is followed by a max-pooling layer and the final two layers are fully connected with 1024 and 10/100 hidden nodes, respectively.  We evaluate the performance of our method on different density levels.

For each experiment, we run Algorithm \ref{alg:adv-winning-ticket} for 20 epochs to find 
AWT and then adversarially train the winning ticket from the original initialization.
\fi

We use PGD attacks for adversarial training and robustness evaluation as suggested in \citet{guo2020meets} and \citet{wang2020achieving}. In practice, $\ell_\infty$ attacks are commonly used and we use adversarial strength $\varepsilon=0.3$ for MNIST and $\varepsilon=8/255$ for CIFAR-10. We take 100 iterations for robustness evaluation, the step size is taken to be $2.5\cdot\varepsilon/100$ as suggested by \citet{madry}. Other detailed experimental configurations such as the learning rate and batch size can be found in the supplementary materials. 

% We left the experiments on larger sized networks as future works, in which we will propose more efficient NTK based algorithms as well as more efficient implementations by using the techniques, such as matrix approximation. At last, we should point out that AWT obtained on very small sized dense networks can be much more robust than the ones obtained by existing methods on much larger sized networks, e.g, ResNet, VGG16. For the details, please refer to Table \ref{tab:accuracy-cifar100} in Section \ref{sec:awt-acc}. This indicates that our method is very promising in adversarial training, and with larger sized networks better performance can be expected.  % this part is mentioned in discussion part

\subsection{Effectiveness in Preservation of Training Dynamics}\label{sec:mtk-var}

% add l2 case here

In this part, we evaluate the ability of our method in preserving the training dynamics. To be precise, at each density level,  we first present the evolution curves of kernel distance and target distance over the whole procedure of finding the adversarial winning ticket. Then we show the adversarial training/testing accuracy during the training process. Since Theorem \ref{thm:main} is valid for any $\ell_p$ attack algorithms, we also present experimental results under $\ell_2$ attacks as well as $\ell_\infty$ attacks.

Figure \ref{fig:kernel-target-distance-AWT-mlp-mnist} (a)/(d) and (b)/(e) show the kernel and target distance curves at different density levels under $\ell_2$/$\ell_\infty$ attack. We can see that as the optimization goes on, both of the kernel and target distances decrease very quickly. As expected, the distance becomes smaller as the density level increases. Figures   \ref{fig:kernel-target-distance-AWT-mlp-mnist} (c)/(f) show the adversarial training/testing accuracy. We can see that when the density becomes larger, the accuracy curve gets closer to the dense one. This indicates that the training dynamics are well preserved.

%------------------------------------------
\begin{table}
\begin{center}
{\footnotesize
%\begin{tabular}{p{1.2cm}<{\centering}p{0.7cm}<{\centering} p{0.7cm}<{\centering}p{0.7cm}<{\centering}p{0.7cm}<{\centering}p{0.7cm}<{\centering}p{0.7cm}<{\centering}}
\begin{tabular}{ccc}
\toprule
density& \citealt{AdvLTH} & AWT \\ \cmidrule(){1-3}
full model& \multicolumn{2}{c}{98.96/91.14}  \\  \cmidrule(r){1-1} \cmidrule(l){2-3}
 51.3\%  &98.07/60.14 &99.13/91.21  \\  \cmidrule(r){1-1} \cmidrule(l){2-3}
 16.9\%  &97.73/59.91 &96.58/89.30  \\  \cmidrule(r){1-1} \cmidrule(l){2-3}
 8.7\%  &97.20/57.60 &94.48/87.51  \\  \cmidrule(r){1-1} \cmidrule(l){2-3}
 3.6\%  &95.58/48.81 &91.74/83.60  \\  \cmidrule(r){1-1} \cmidrule(l){2-3}
 1.8\%  & 92.67/38.23 &87.69/78.66   \\  
 
 \bottomrule
\end{tabular}
}
\end{center}
\caption{Test accuracy on natural/adversarial examples over different density levels on MNIST with MLP.} \label{tab:accuracy-lenet-mnist}
%\vspace*{-10pt}
\end{table} 
%------------------------------------------

To verify the quality of our winning ticket, we compare our method with the latest work by \citet{AdvLTH}, which finds the winning ticket by iteratively pruning and adversarial training. As indicated by the authors \cite{{AdvLTH}}, their method is computationally expensive so they only evaluated it on small MLPs. Therefore we only give the comparison result on MNIST with MLPs here. Specifically, we give the test accuracy on natural/adversarial examples in Table \ref{tab:accuracy-lenet-mnist}.  It shows that our method can outperform the baseline with a large margin. For example, at the density of $1.8\%$, the adversarial test accuracy of our method is $40\%$ higher than that of \citealt{AdvLTH}. We can also see that the accuracy of our method can converge to the dense model much more quickly than the baseline as the density increases. This is benefited from the dynamics preserving property of our sparse sub-network structure.

\subsection{Robustness of Trained Sparse Networks}\label{sec:awt-acc}

In this section, we evaluate the robustness of fully adversarially trained AWT at different density levels. 

\begin{figure}[t!]
    \centering
    \includegraphics[width=0.48\textwidth]{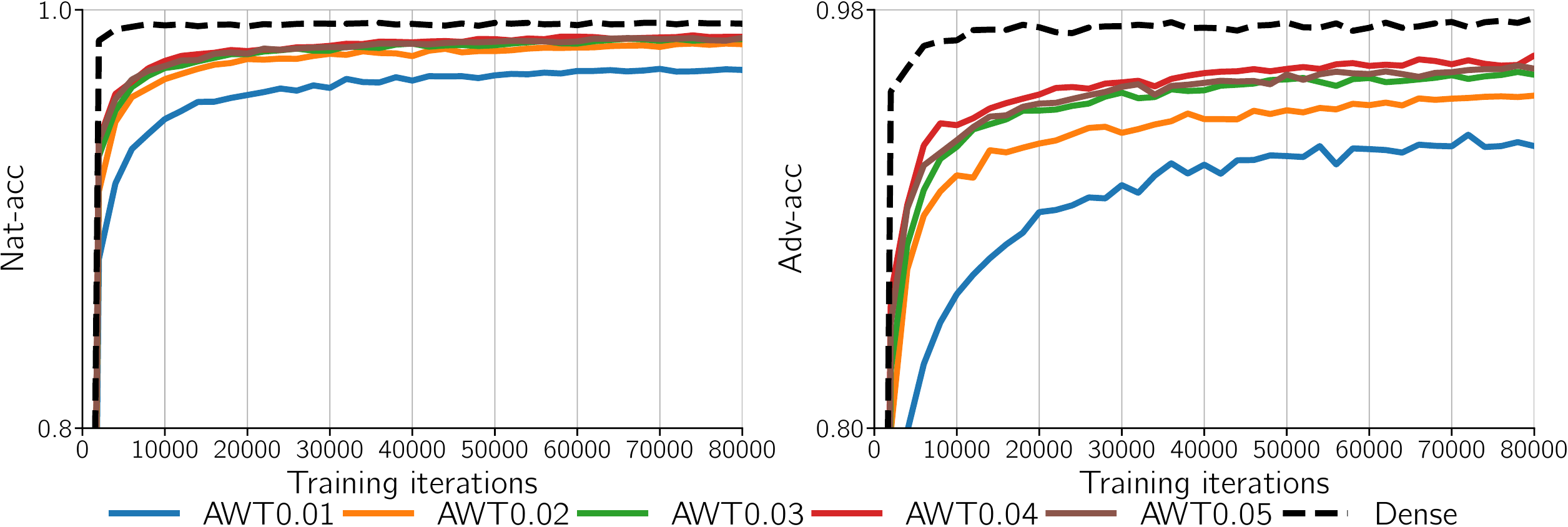}
    \caption{Test accuracy on natural and adversarial examples of CNN trained on MNIST. The density varies in $\{0.01, 0.02, 0.03,0.04, 0.05\}$.}
    \label{fig:acc-mnist-cnn-001-005}

\end{figure}

\begin{figure}[t!]
    \centering
    \includegraphics[width=0.48\textwidth]{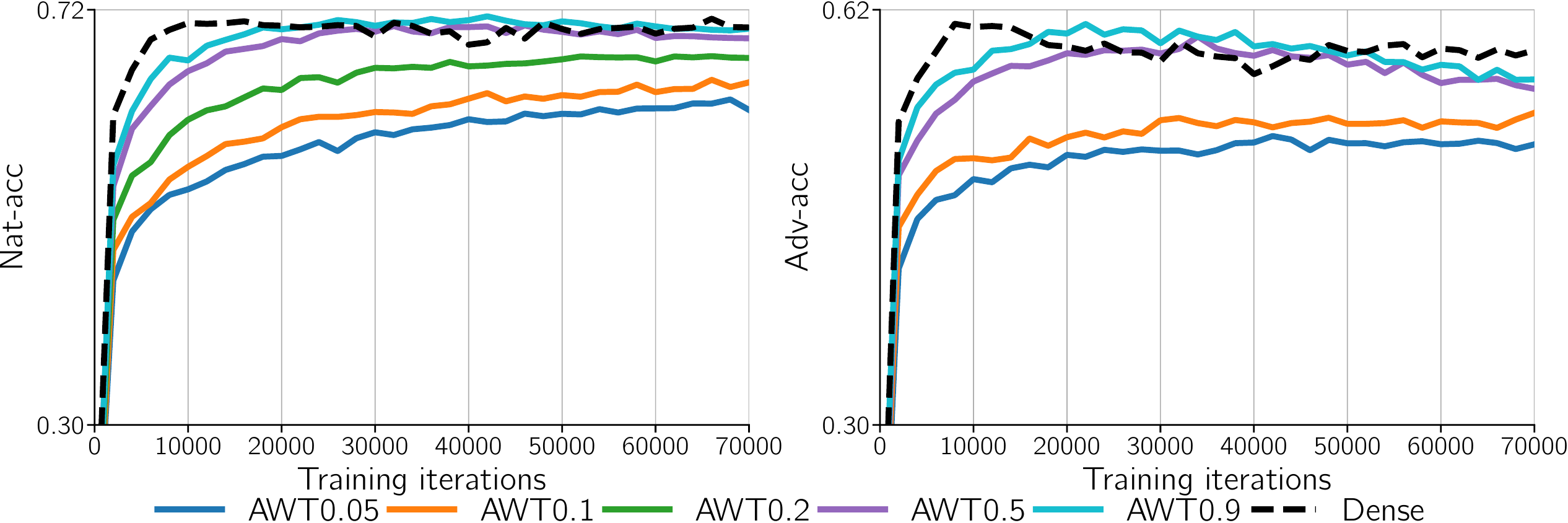}
    \caption{Test accuracy on natural and adversarial examples of CNN trained on CIFAR10.}
    \label{fig:acc-CIFAR10}
\end{figure}

We first present the test accuracy on natural and adversarial examples of the CNN models trained on MNIST and CIFAR10. For CIFAR10, the density varies in $\{0.05, 0.1, 0.2,0.3,\ldots, 0.9\}$. For MNIST, since it  can be classified with much sparser networks compared with CIFAR10, in this section, we only check densities between $\{0.01,\ldots, 0.05\}$ and the results under higher density levels can be found in the appendix. Figure \ref{fig:acc-mnist-cnn-001-005} and \ref{fig:acc-CIFAR10} give the results on MNIST and CIFAR10, respectively. Both of these two Figures show that the models trained by our method have high natural and adversarial test accuracy even when the model is very sparse. For example, Figure \ref{fig:acc-mnist-cnn-001-005} shows that at the density of 0.03, the model trained by our method can reach the test accuracy of $0.98$ and $0.96$ on natural and adversarial examples, which are quite close to the dense model. We can also see that the training dynamic, i.e., the test accuracy curves, can converge to that of the dense model as the density increases. And the sparse models obtained by our method can achieve comparable test accuracy with the dense model after trained with the same number of epochs.

\begin{figure}[htb!]
    \centering
    \includegraphics[width=0.48\textwidth]{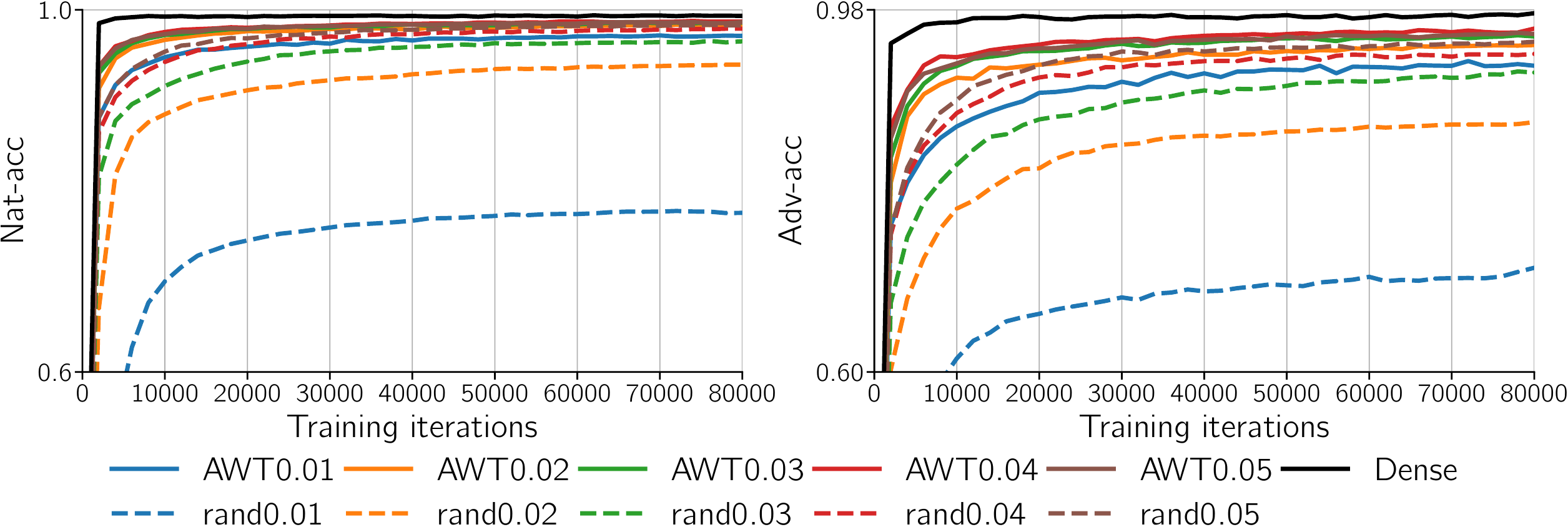}
    \caption{Natural and adversarial test accuracy of the models trained from AWT and random structure on MNIST with CNN.}
    \label{fig:acc-mnist-cnn-001-005-AWT-rand}
\end{figure}

We then compare the performance of models trained from our sparse structure and the random structure at the same density level. We give the results of CNN trained on MNIST with the density varies in $\{0.01, 0.02, \ldots, 0.05\}$ in Figure \ref{fig:acc-mnist-cnn-001-005-AWT-rand}. More results can be found in the appendix. From Figure \ref{fig:acc-mnist-cnn-001-005-AWT-rand}, our method can achieve much higher natural and adversarial test accuracy than the models trained from the random structure. For example, the network trained by our method at the density of $1\%$ achieves higher adversarial test accuracy than the model trained from random structure at the density of $3\%$. It also shows that the learning curves of our method are  much closer to the dense model than the random structure. This demonstrates that our sparse structures indeed have similar training dynamic with the dense model.   

\subsection{Discussion on Scalability}\label{sec:scale}
In the same situation as existing works on NTK, the expensive computational cost of NTK hinders us from conducting large-scale experiments. Fortunately, the following preliminary experiment shows that the methods such as sampling on the NTK matrix could be promising to improve the efficiency of our method. 

To be specific, inspired from Jacobi preconditioner method \citep{concus1976generalized} in optimization theory, we  sample only the diagonal elements in the MTK matrices $\Theta_0(X, \tilde{X}_0)$ and $\Theta_0^s(X, \tilde{X}^s_0)$ in Equation \eqref{eq:awt-op} and keep all other settings the same as above. In this way, the computational complexity of training can be significantly reduced. We conduct a preliminary experiment on CIFAR-10 with large-sized model VGG-16. To the best of our knowledge, we are the first to apply NTK into models as large as VGG-16.  The result is presented in Table \ref{tab:vgg}. It shows that the performance degradation caused by the sampling is unnoticeable. We will investigate this kind of approaches to improve the efficiency of NTK based methods in the future.

\begin{table}[htb!]
\begin{center}
{\footnotesize
%\begin{tabular}{p{1.2cm}<{\centering}p{0.7cm}<{\centering} p{0.7cm}<{\centering}p{0.7cm}<{\centering}p{0.7cm}<{\centering}p{0.7cm}<{\centering}p{0.7cm}<{\centering}}
\begin{tabular}{cccc}
\toprule
density   & WT & S-AWT & IWI \\ \cmidrule(){1-4}
10\%   & 40.10 & 46.68 & 47.53\\ \cmidrule(){1-4}
 5\%  & 45.15 & 45.87 & 47.59 \\
 \bottomrule
\end{tabular}
}
\end{center}
\caption{Adversarial test accuracy of the VGG16 trained on CIFAR10 at different density levels. WT represents winning ticket, which applies iteratively pruning and adversarial training method. S-AWT is AWT with sampling.  IWI represents Inverse Weight Inheritance. WT and IWI results are copied from \citet{wang2020achieving}.} \label{tab:vgg}
\end{table}

\section{Conclusions}\label{sec: discussion}
We study the evolution of adversarially trained networks and obtain a new type of kernel quantifying the dynamics of adversarial training. We verify the \textit{Lottery Ticket Hypothesis} (LTH) \citep{LTH} in adversarial setting by solving an optimization 
problem to find adversarial winning ticket (AWT), which can be adversarially trained to be robust from scratch. Our work includes the classical LTH in natural training as a special case and extends the bound of \textit{Neural Tangent Transfer} (NTT) \citep{NTKLTH}. Unlike most of the adversarial pruning methods, which follow the classical pruning pipelines, i.e., prune the network during training, our method is a foresight pruning method. To the best of our knowledge, this is the first result showing that to identify winning tickets  at initialization is possible in the adversarial training scenario.

We would like to point out our main contribution  is the verification of LTH in adversarial setting, rather than proposing a practical pruning method. Similar to existing NTK based methods, We did not conduct experiment on large scale datasets, such as ImageNet, due to the well-known fact that the computation of NTK is expensive. Moreover, the preliminary experimental result in Section \ref{sec:scale} indicates certain approximation methods such as sampling can be promising to reduce the computational cost of NTK without noticeable performance degradation. We will further explore this possibility in our future work.

%We propose a way of finding adversarial winning ticket by studying the dynamics of adversarial training. AWT is a foresight pruning method which achieve comparable adversarial robustness as dense adversarial training. Our work is inspired by \cite{NTKLTH}, which finds trainable sparse sub-network for natural training. We consider in the adversarial context, and theoretically show the lottery ticket hypothesis proposed in \cite{LTH} also exists for adversarial training. One of the limitations of the present work is that the calculation of NTK matrices is computationally expensive.  We are excited about the possibility of exploring more efficient ways to compute the NTK in the future. %V1 

%\clearpage

{
	\small
	\bibliographystyle{plainnat}
	\bibliography{awt}
}

%%%%%%%%%%%%%%%%%%%%%%%%%%%%%%%%%%%
%%%%%% SUPPLEMENT (OPTIONAL) %%%%%%
%%%%%%%%%%%%%%%%%%%%%%%%%%%%%%%%%%%

\clearpage
%\appendix

%\thispagestyle{empty}

% For one-column format, uncomment the following:
%\onecolumn \makesupplementtitle
% For two-column format, uncomment the following:
%\twocolumn[ \makesupplementtitle ]

\begin{appendix}

\setcounter{equation}{0}
\renewcommand{\theequation}{\thesection.\arabic{equation}}

\onecolumn
\aistatstitle{Supplementary Materials: \\
Finding Dynamics Preserving Adversarial Winning Tickets}

\section{Proof of Theorem}\label{sec:supp-theory}
Let $\mathcal{D}=X\times Y=\{(x_1,y_1),\cdots, (x_N,y_N)\}$ be the empirical data distribution, $f_\theta(x)\in\mathbb{R}^{k}$ the network function,  and $f_\theta(X)=\mathrm{vec}\big([f_\theta(x)]_{x\in X}\big)\in\mathbb{R}^{k|\mathcal{D}|\times 1}$ be the model outputs. Note that adversarial training optimizes the following objective function
\begin{equation}\label{eq:supp-adv-op}
\min_\theta \mathcal{L}=\mathbb{E}_{(x,y)\sim\mathcal{D}}\max_{r\in S_\varepsilon(x)}\ell(f_\theta(x+r), y) 
\end{equation}
We use the following notation for convenience:
\begin{equation*}
    \nabla_f\mathcal{L}(X)=\begin{bmatrix}
    |\\ \nabla_f\ell(f(x_i), y)\\|
    \end{bmatrix}
\end{equation*}
For squared loss $\ell(f(x),y)=\displaystyle\frac{1}{2}\norm{f(x)-y}_2^2$, this is just
\begin{equation*}
    \nabla_f\mathcal{L}(X)= \begin{bmatrix}
   |\\f(x_i)-y_i\\|
    \end{bmatrix}=f(X)-Y
\end{equation*}
\begin{thm}\label{thm:supp-dynamics}
Let $f_t(x):=f_{\theta_t}(x)$ be the timely dependent network function and $\tilde{X}_t$ the adversarial examples generated at time $t$. Then the continuous gradient descent of adversarial training is:
\begin{equation}\label{eq:supp-continuous-sgd}
\begin{split}
    \frac{\dd \theta_t}{\dd t}=-\frac{\eta}{N}\nabla_\theta^T f_t(\tilde{X}_t)\nabla_{f}\mathcal{L}(\tilde{X}_t)
\end{split}
\end{equation}
As a result, $f_t$ satisfies the following differential equation:
\begin{equation}\label{eq:supp-adv-dyn}  
\begin{split}
    \frac{\dd f_t}{\dd t}(X) & =\nabla_\theta f_t(X)\frac{\dd\theta_t}{\dd t}\\
    & =-\frac{\eta}{N}\nabla_\theta f_t(X)\nabla_\theta^T f_t(\tilde{X}_t)\nabla_{f}\mathcal{L}(\tilde{X}_t)
\end{split}
\end{equation}
\end{thm}
\begin{proof}
Note that at time $t$ adversarial training consists of an attack step and parameter update step. To be precise, for some chosen strong attack algorithm, we first generate the set of adversarial examples $\tilde{X}_t$, then update the parameter according to
\begin{equation}\label{eq:supp-sgd-dt}
\begin{split}
    \theta_{t+dt} & =\theta_t-\eta_t\frac{\partial\mathcal{L}}{\partial\theta_t}(\tilde{X}_t)\\
    & =\theta_t-\frac{\eta_t}{N}\nabla_\theta^Tf_t(\tilde{X}_t)\nabla_{f}\mathcal{L}(\tilde{X}_t)
\end{split}
\end{equation}
If we take the infinitesimal learning rate to be\footnote{The discrete parameter update corresponds to the case when $dt=1$.} $\eta_t=\eta dt$ and taking the limit $dt\to 0$, we obtain the continuous gradient descent as in Equation \eqref{eq:supp-continuous-sgd}. The evolution of $f_t$ in Equation \eqref{eq:supp-adv-dyn} is a direct result by chain rule. 
\end{proof}
We consider the following optimization problem to find AWT:
\begin{equation}\label{eq:supp-awt-op}
    \min_{m}\mathcal{L}_{awt}=
        \frac{1}{N}\norm{f_0(\tilde{X}_0)-f_0^s(\tilde{X}^s_0)}^2+\frac{\gamma^2}{N^2}\norm{\Theta_0(X, \tilde{X}_0)-\Theta^s_0(X, \tilde{X}_0^s)}_F^2
\end{equation}
where $\Theta(X,Y)=\nabla_\theta f(X)\nabla_\theta f^T(Y)$ is the empirical neural tangent kernel and sup-script $s$ represents quantities involving sparse structure. 
\begin{thm}[\textbf{Existence of adversarial winning ticket}]\label{thm:supp-main}
Let $f_\theta(x)$ denote the dense network function. Suppose $f_\theta$ has identical number of neurons for each layer, i.e. $n_1=n_2=\cdots=n_L=n$ and assume $n$ is large enough. Denote $f^s_{m\odot\theta}(x)$ the corresponding sparse network with $1-p$ weights being pruned. Assume $f$ and $f^s$ have bounded first and second order derivatives with respect to $x$, i.e.
\begin{align*}
    &\max_{t, x}\big\{\norm{\partial_{x} f_t}_q, \norm{\partial_{x} f^s_t}_q\big\}\le C_{1,q}\\
    &\max_{t, x}\big\{\norm{\partial_{xx}^2f_t}_{p,q}, \norm{\partial_{xx}^2f^s_t}_{p,q}\big\}\le C_{2,q}
\end{align*}
where we choose an $\ell_p$ attack to generate adversarial examples such that $q$ is the conjugate of $p$ in the sense of $1/p+1/q=1$.\footnote{If $p=\infty$, we take $q=1$.} Denote the optimal loss value for AWT optimization problem \eqref{eq:awt-op} to be $\mathcal{L}_{awt}^*=\alpha^2$. Then for all $t\le T$ with $T$ the stop time, with learning rate $\eta=O(T^{-1})$, we have
\begin{equation}\label{eq:supp-main}
    \underset{{x\in\mathcal{D}}}{\mathbb{E}}\norm{f_t(x)-f^s_t(x)}^2\le 4(\alpha+4C_q\varepsilon)^2
\end{equation}
where $C_q=C_{1,q}+\varepsilon C_{2,q}$ is a constant. 
\end{thm}
In order to prove the theorem, we need the following lemma of estimation of error bound.

\begin{lem}\label{lem:bound-value}
For any $1<p\le\infty$, assume an $k$ iterative $\ell_p$ attack algorithm updates as $\tilde{x}_0=x,\  \tilde{x}_t=\tilde{x}_{t-1}+r_t$ with $\norm{r_t}_p\le\delta$ for any $1\le t\le k$ and with total allowed perturbation strength $\norm{\sum r_j}_p\le\varepsilon$. Assume $k\delta\le 2\varepsilon$. If the neural network function $f$ has bounded first and second order derivative with respect to $x$, i.e. $\norm{\partial_{x} f}_q\le C_{1,q}, \norm{\partial_{xx}^2f}_{p,q}\le C_{2,q}$, where $q$ be the conjugate of $p$ such that $1/p+1/q=1$. Then for any adversarial example $\tilde{x}$ generated by the attack algorithm, we have
\begin{equation}\label{eq:bound}
    |f(\tilde{x})-f(x)|\le 2\varepsilon C_{1,q}+2\varepsilon^2C_{2,q}=2\varepsilon C_q
\end{equation}
\end{lem}
\begin{proof}[Proof of Lemma:]
Consider the series of second order Taylor expansions for any $1\le t\le k$
\begin{equation}
    f(\tilde{x}_t)-f(\tilde{x}_{t-1})=\partial_x f(\tilde{x}_{t-1})r_t+\frac{1}{2}r_t^T\partial^2_{xx}f(\xi_{t-1})r_t
\end{equation}
We have
\begin{equation}
\begin{split}
        |f(\tilde{x}_{t})-f(\tilde{x}_{t-1})|
    \le & \norm{\partial_xf(\tilde{x}_{t-1})r_t}+\norm{\frac{1}{2}r_t^T\partial^2_{xx}f(\xi_{t-1})r_t}\\
    \le & \norm{\partial_xf(\tilde{x}_{t-1})}_q\norm{r_t}_p+\frac{1}{2}\norm{r_t}_p\norm{\partial^2_{xx}f(\xi_{t-1})r_t}_q\\
    \le & \norm{\partial_xf(\tilde{x}_{t-1})}_q\delta+\frac{1}{2}\delta\norm{\partial^2_{xx}f(\xi_{t-1})}_{p,q}\norm{r_t}_p\\
    \le & C_{1,q}\delta+\frac{1}{2}C_{2,q}\delta^2
\end{split}
\end{equation}
where we use H\"{o}lder inequality in the second step and definition of $(p,q)$ norm in the third step. On the other hand, by Mean-value theorem we have
\begin{equation}
    \partial_xf(\tilde{x}_{t})=\partial_xf(\tilde{x}_{t-1})+\partial^2_{xx}f(\eta_{t-1})r_{t}
\end{equation}
Hence we have
\begin{equation}
    \begin{split}
        \norm{\partial_xf(\tilde{x}_{t})}_q&\le \norm{\partial_xf(\tilde{x}_{t-1})}_q+\norm{\partial^2_{xx}f(\eta_{t-1})r_{t}}_q\\
        &\le\norm{\partial_xf(\tilde{x}_{t-1})}_q+\norm{\partial^2_{xx}f(\eta_{t-1})}_{p,q}\norm{r_t}_p\\
        &\le C_{1,q}+C_{2,q}\delta
    \end{split}
\end{equation}
where we use Minkowski inequality in the first step. Together, we have the following estimation:
\begin{equation}
    \begin{split}
        |f(\tilde{x})-f(x)|
        =&|f(\tilde{x}_k)-f(\tilde{x}_0)|\\
        \le&\sum_{t=1}^k |f(\tilde{x}_t)-f(\tilde{x}_{t-1})|\\
        \le&\delta\sum_{t=1}^k\norm{\partial_xf(\tilde{x}_{t-1})}_q+\frac{1}{2}\delta^2\sum_{t=1}^k\norm{\partial^2_{xx}f(\xi_{t-1})}_{p,q}\\
        =&\delta\bigg(\sum_{t=1}^k\norm{\partial_xf(\tilde{x}_0)}_q+C_{2,q}\delta (t-1)\bigg)+\frac{1}{2}\delta^2kC_{2,q}\\
        =&k\delta C_{1,q}+\frac{1}{2}k^2\delta^2C_{2,q}
    \end{split}
\end{equation}
Then Equation \eqref{eq:bound} is valid if we plug in the assumption $k\delta\le 2\varepsilon$.
\end{proof}
\paragraph{Remark 1:} We did not specify any particular algorithm in the presentation of our lemma. In practice, $k$ steps PGD attacks is the common choice for inner subproblem of adversarial training. For $k$ steps PGD attack, the number of attack iteration is usually taken to be $7$ (ImageNet) or $20$ (MNIST/CIFAR-10), so we may assume $k=\Omega(1)$ for future use. The assumption $k\delta\le 2\varepsilon$ is also for practical consideration, where we usually choose the step size $\delta$ approximately to be $2\varepsilon/k$ as suggested in \cite{madry}.

\paragraph{Remark 2:} We might get more accurate bound by looking deeper into the dynamics of continuous first-order attack:
\begin{equation}
    \frac{\dd x_t}{\dd t}=\frac{\dd \ell}{\dd x}(f_t(x_t)),\quad\quad x_0=x
\end{equation}
The analysis of the above differential equations would possibly weaken the current assumption on derivatives. We would leave this as a future work.

Now we are ready to prove the main theorem.
\begin{proof}[Proof of Theorem]
Suppose $\mathcal{L}_{awt}^*=\alpha^2$, then we have
\begin{align*}
    \norm{f_0(\tilde{X}_0)-f_0^s(\tilde{X}_0^s)}
    \le \sqrt{N}\alpha,\quad\quad\norm{\Theta_0(X,\tilde{X}_0)-\Theta_0^s(X,\tilde{X}^s_0)}_F\le N\frac{\alpha}{\gamma}
\end{align*}
To bound the distance between dense and sparse output, we do induction on time $t$ and prove the following stronger estimation
\begin{equation}
    \norm{f_{t}(X)-f_{t}^s(X)}\le \bigg(1+\frac{t}{T}\bigg)\sqrt{N}\big(\alpha+4C_q\varepsilon\big)
\end{equation}
where $C_q=C_{1,q}+\varepsilon C_{2,q}$ and $C_{1,q}$ and $C_{2,q}$ are bounds of first and second order derivative of $f$. Note that at $t=0$, we have
\begin{equation}
\begin{split}
    \norm{f_0(X)-f_0^s(X)}
    \le & \norm{f_0(\tilde{X}_0)-f_0^s(\tilde{X}_0^s)}+\norm{f_0(\tilde{X}_0)-f_0(X)}+\norm{f_0^s(\tilde{X}^s_0)-f_0^s(X)}\\
    \le & \sqrt{N}\big(\alpha+4C_q\varepsilon\big)
\end{split}
\end{equation}
At time $t$, assume we have
\begin{equation}
    \norm{f_{t}(X)-f_{t}^s(X)}\le \bigg(1+\frac{t}{T}\bigg)\sqrt{N}\big(\alpha+4C_q\varepsilon\big)
\end{equation}
Then according to the dynamical equation \eqref{eq:supp-adv-dyn}, we have
\begin{equation}
\begin{split}
    f_{t+1}(X)&=f_t(X)-\frac{\eta}{N}\Theta(X,\tilde{X}_t)\big(f_t(\tilde{X}_t)-Y\big)\\
    f_{t+1}^s(X)&=f^s_t(X)-\frac{\eta}{N}\Theta^s(X,\tilde{X}^s_t)\big(f^s_t(\tilde{X}^s_t)-Y\big)\
\end{split}
\end{equation}
Then
\begin{equation}\label{eq:t-bound}
        \norm{f_{t+1}(X)-f^s_{t+1}(X)}
    \le \norm{f_{t}(X)-f_{t}^s(X)}+\frac{\eta}{N}\norm{(\Theta_t-\Theta_t^s)(f_t(\tilde{X}_t)-Y)} +\frac{\eta}{N}\norm{\Theta_t^s(f_t(\tilde{X}_t)-f_t^s(\tilde{X}_t^s))}
\end{equation}
Let $K_t(X,X)=\nabla_\theta f_t(X)\nabla_\theta^Tf_t(X)$ be the empirical neural tangent kernel at time $t$, then according to Theorem 2.1 in \citet{EvolveLinear}, $\norm{K_t-K_0}_F=\frac{C_K}{\sqrt{n}}$, where $n$ is the number of neurons in each layer. Therefore we have
\begin{equation}
\begin{split}
    &\norm{\Theta_t(X,\tilde{X}_t)-\Theta_t^s(X,\tilde{X}^s_t)}_F\\
\le &\norm{\Theta_t(X,\tilde{X}_t)-\Theta_0(X,\tilde{X}_0)}_F
+\norm{\Theta_t^s(X,\tilde{X}^s_t)-\Theta_0^s(X,\tilde{X}_0)}_F
+\norm{\Theta_0(X,\tilde{X}_0)-\Theta_0^s(X,\tilde{X}_0^s)}_F\\
\le&\norm{K_t(X,X)-K_0(X,X)}_F+\norm{K_t^s(X,X)-K_0^s(X,X)}_F 
+\norm{\nabla_\theta f_t(X)\nabla_{\theta,x}^2f(\xi)R_t} \\
&+\norm{\nabla_\theta f^s_t(X)\nabla_{\theta,x}^2f(\tilde{\xi})R^s_t}_F
+ N\frac{\alpha}{\gamma}\\
\le & \frac{C_K}{\sqrt{n}}\bigg(1+\frac{1}{\sqrt{p}}\bigg)+2N\varepsilon C_{1,q} C_{2,q}+N\frac{\alpha}{\gamma}
\end{split}
\end{equation}
And by Lemma \eqref{lem:bound-value}, we have
\begin{equation}
\begin{split}
    \norm{f(\tilde{X}_t)-f^s(\tilde{X}_t^s)}
    \le &
    \norm{f_t(X)-f_t^s(X)}+\norm{f(\tilde{X}_t)-f_t(X)}
    +\norm{f^s_t(\tilde{X}_t)-f_t^s(X)}\\
    &\le  \bigg(1+\frac{t}{T}\bigg)\sqrt{N}\big(\alpha+4C_q\varepsilon\big)+4\sqrt{N}C_q\varepsilon
\end{split}
\end{equation}
Then Equation \eqref{eq:t-bound} reads
\begin{equation}\label{eq:ugly}
    \begin{split}
        &\norm{f_{t+1}(X)-f^s_{t+1}(X)}\\
    \le & \bigg(1+\frac{t}{T}\bigg)\sqrt{N}\big(\alpha+4C_q\varepsilon\big)
    + \frac{\eta}{N}\bigg(\frac{C_K}{\sqrt{n}}\bigg(1+\frac{1}{\sqrt{p}}\bigg)
    +2N\varepsilon C_{1,q} C_{2,q}+N\frac{\alpha}{\gamma}\bigg)\sqrt{N}c\\
    &+\frac{\eta}{N}NC_{2,q}\bigg[\bigg(1+\frac{t}{T}\bigg)\sqrt{N}\big(\alpha+4C_q\varepsilon\sqrt{N}\big)+4\sqrt{N}C_q\varepsilon\bigg]
    \end{split}
\end{equation}
where $c=\max\{|f(x)-y|:x\in X\}$ is bounded essentially. Take
\begin{equation}
    \eta=\min\bigg\{\frac{1}{T(2+\frac{c}{\gamma})},\quad \frac{4C_q\varepsilon}{T\big(2cC_{1,q}C_{2,q}+8C_q)\big)}\bigg\}
\end{equation}
One can check that if $N$ is sufficiently large such that $\frac{C_K}{\sqrt{Nn}(1+\frac{1}{\sqrt{p}})}\to 0$, then Equation \eqref{eq:ugly} reads
\begin{equation}
    \norm{f_{t+1}(X)-f^s_{t+1}(X)}\le \bigg(1+\frac{t+1}{T}\bigg)\sqrt{N}\big(\alpha+4C_q\varepsilon\big)
\end{equation}
which completes the proof.
\end{proof}

\section{Discussion on Possible Extensions to Other Loss Functions}\label{sec:extensions}
We only consider the squared loss for simplicity in our main theorem. However, it is possible to extend our result to the cross-entropy loss case. Actually, we use cross-entropy loss in our experiment, so we have already checked our method empirically.

To see how our method works theoretically for cross-entropy loss, let 
\begin{equation}\label{lem:cross-entropy}
    \ell_{ce}(f(x), y)=-\log\frac{e^{f_y}}{\sum_i e^{f_i}}
\end{equation}
be the cross-entropy loss, where $f(x)=[\cdots,f_j(x),\cdots]^T$ is the model output. Note that we have
\begin{equation}
    \frac{\partial\ell_{ce}}{\partial f_y}=\frac{e^{f_y}}{\sum_i e^{f_i}}-1,\quad\quad
    \frac{\partial\ell_{ce}}{\partial f_j}=\frac{e^{f_j}}{\sum_i e^{f_i}}, j\neq y
\end{equation}
Therefore,
\begin{equation}
    \norm{\nabla_f\ell_{ce}(f(x),y)-\nabla_{f^s}\ell_{ce}(f(x),y)}^2
    =\sum_j\bigg(\frac{e^{f_j}}{\sum_i e^{f_i}}-\frac{e^{f^s_j}}{\sum_i e^{f^s_i}}\bigg)^2
\end{equation}
In practice, we may expect the change of model outputs is usually larger than the outputs after cross-entropy loss, i.e. 
\begin{equation}
    \forall j, \bigg(\frac{e^{f_j}}{\sum_i e^{f_i}}-\frac{e^{f^s_j}}{\sum_i e^{f^s_i}}\bigg)^2\le [f_i(x)-f^s_i(x)]^2 
\end{equation}
Hence, this implies 
\begin{equation}\label{eq:bound-ce}
    \norm{\nabla_f\ell_{ce}(f(x),y)-\nabla_f\ell_{ce}(f^s(x),y)}\le \norm{f(x)-f^s(x)}
\end{equation}
Recall that the dynamics of adversarial training is
\begin{equation} 
     \frac{\dd f_t}{\dd t}(X)
    =-\eta\Theta_t(X, \tilde{X}_t)\nabla_{f_t}\mathcal{L}(\tilde{X}_t)
\end{equation}
We may expect the optimization problem of finding adversarial winning ticket for general loss function is
\begin{equation}\label{eq:op-ce}
    \min_m\mathcal{L}_{awt}'=\frac{1}{N}\norm{\nabla_{f}\mathcal{L}(\tilde{X}_0)-\nabla_{f^s}\mathcal{\tilde{L}}(\tilde{X}_0')}
    +\frac{1}{N^2}\norm{\Theta_0(X,\tilde{X}_0)-\Theta^s_0(X,\tilde{X}_0')}_F
\end{equation}
Now we compare the optimization problem \eqref{eq:op-ce} with problem \eqref{eq:supp-awt-op}. The first term of problem \eqref{eq:op-ce}, as discussed above, is bounded by $\norm{f_0(X+R_0)-f^s_0(X+R^s_0)}$. This shows that, if we find adversarial winning ticket according to optimization problem \eqref{eq:op-ce}, the resulting training dynamics is bounded by the one we obtained before, so is close to the dynamics of dense network also. This suggests that the optimization problem \eqref{eq:supp-awt-op} is general for both squared loss and cross-entropy loss.

\textbf{Remark:} We may expect that for many other loss functions, the following condition is true
\begin{equation}\label{eq:appendix.cond}
\norm{\nabla_f\ell(f(x),y)-\nabla_{f^s}\ell(f^s(x),y)}\le C\norm{f(x)-f^s(x)}  
\end{equation}
Our method generalizes to the case using any loss function satisfying condition \eqref{eq:appendix.cond}. Actually, for any convex loss function $\ell$, if the second order derivative is bounded, then the above condition \eqref{eq:appendix.cond} is true by Taylor expansion.
\begin{figure*}[htb!]
    \centering
    \subfigure[]{\includegraphics[scale=0.28]{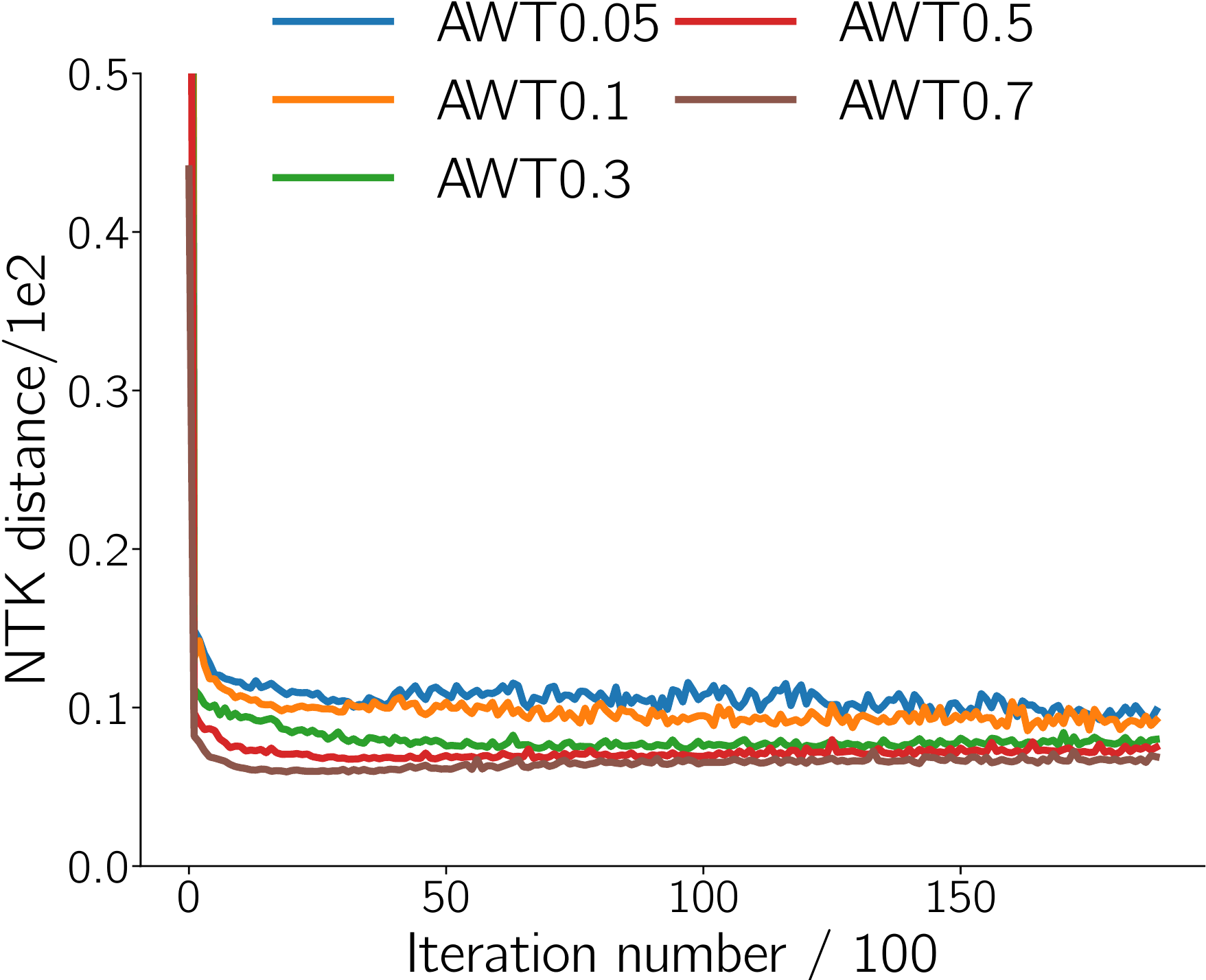}}
    \subfigure[]{\includegraphics[scale=0.28]{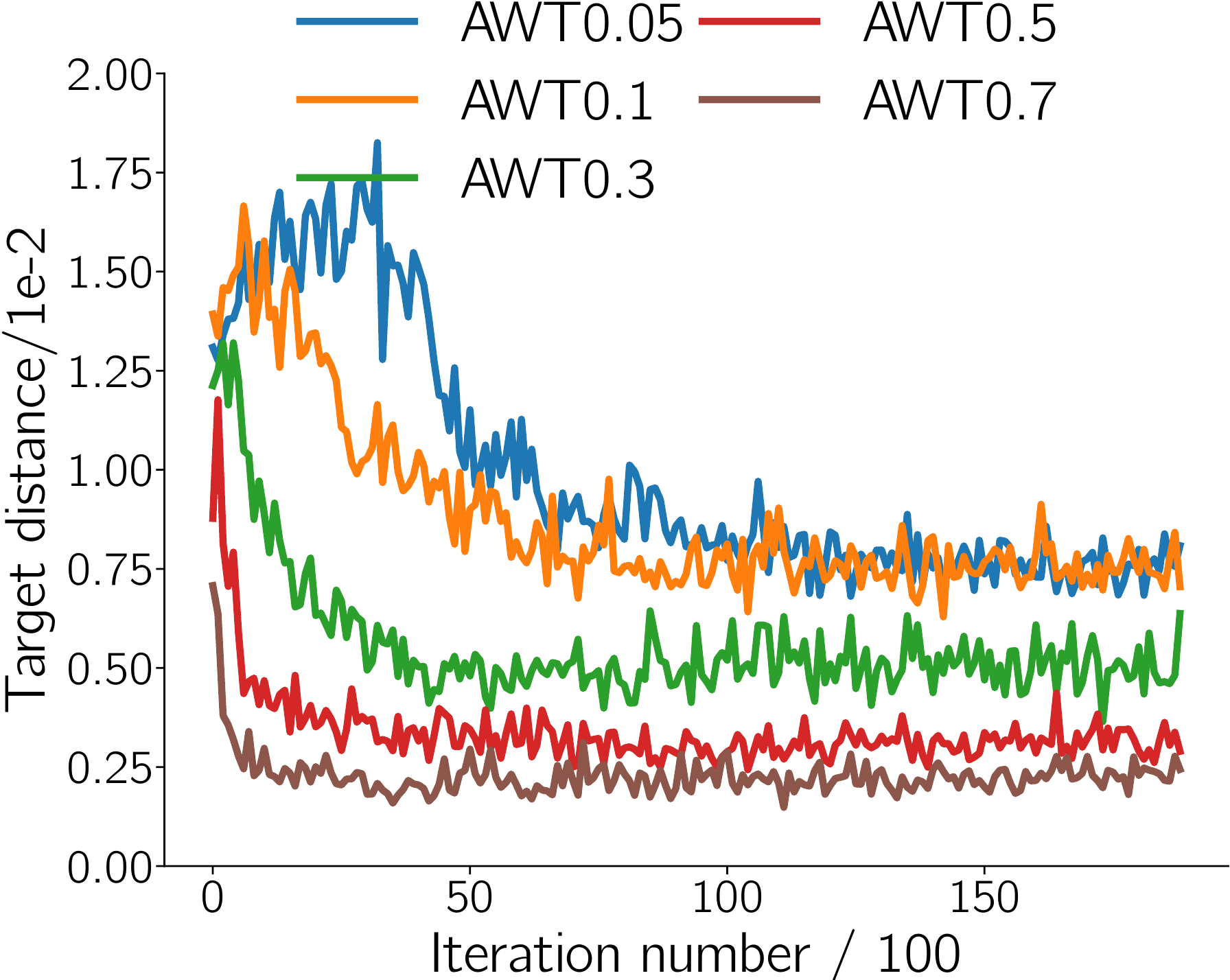}}
     \subfigure[]{\includegraphics[scale=0.28]{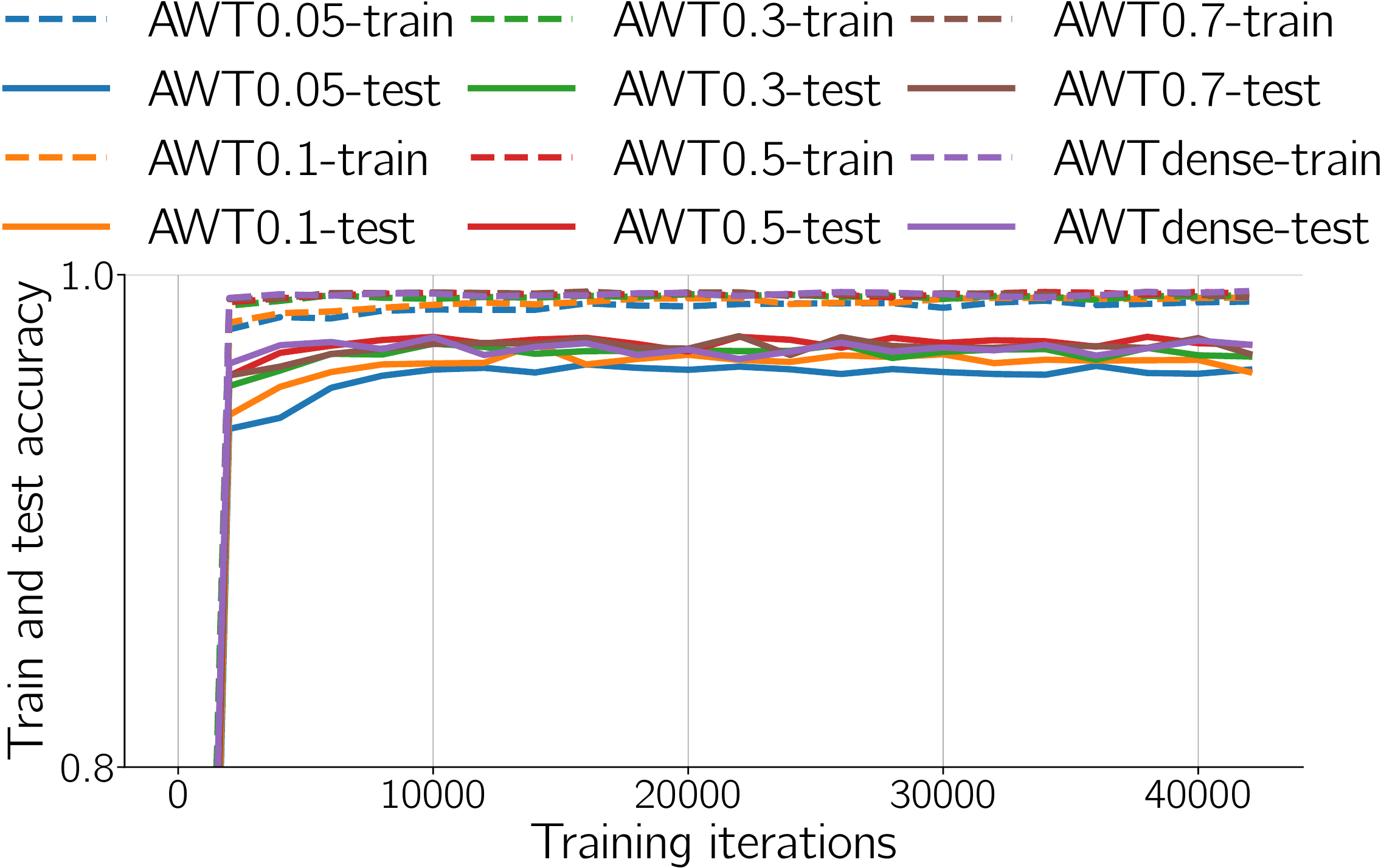}}
    \caption{(a) and (b) are the statistics of AWT in the procedure of searching masks on MNIST with CNN over different density levels. (c) presents the adversarial training and test accuracy curves in the training process}
    \label{fig:kernel-target-distance-AWT-cnn-mnist}
\end{figure*}
\section{More Experimental Results}\label{sec:more-experiment-results}

\subsection{Toy Example}\label{sec: toy}
We provide here a toy example to illustrate our method. Consider a binary classification problem of mixed Gaussian distribution $p(x)=0.5\mathcal{N}(\mu_+,\Sigma_+)+0.5\mathcal{N}(\mu_-,\sigma_-)$ and any linear classifier $C(x)=\mathrm{sgn}(f_\theta(x))=\mathrm{sgn}(\theta\cdot x)$. Let $\mu=0.5(\mu_+-\mu_-)$ be the mean difference and suppose $\Sigma_+=\Sigma_-=\sigma^2I_d$, then the adversarial accuracy with given adversarial perturbation strength $\varepsilon$ can be calculated explicitly using the following result in \cite{xupeng}:
\begin{equation}\label{eq: p-adv}
\begin{split}
     p_{adv-acc}&=1-p_{adv}=p_m+\Phi\bigg(\frac{\theta\cdot\mu}{\norm{\theta}\sigma}-\frac{\varepsilon}{\sigma}\bigg)\\
     &=p_m+\Phi\bigg(\frac{\norm{\mu}}{\sigma}\cos\gamma-\frac{\varepsilon}{\sigma}\bigg)
\end{split}
\end{equation}
where $p_m$ is the misclassification rate, $p_{adv}$ is the probability of the existence of adversarial examples, and $\Phi$ is the cumulative density function of standard normal distribution. This shows the adversarial robustness of $C(x)$ can be measured explicitly by the deflection angle $\gamma$ between $\theta$ and $\mu$. In particular, $p_{adv-acc}$ attains its maximum at $\theta=\mu$, i.e. Bayes classifier is the best classifier in the sense of adversarial robustness. 
 
To illustrate our idea, we randomly generate 5000 data points from $p(x)$, where $\mu_+=[3,0,\cdots,0]=-\mu_-$, $\Sigma_+=\Sigma_-=I_d$ and the dimension is $d=100$. We minimize the loss in equation \eqref{eq:supp-awt-op} to obtain a sparse robust structure with a sparsity level at 10\%, i.e. we only keep 10 nonzero coordinates of $\theta$. Then we adversarially train the AWT to get the sparse robust network. We use SVM as baseline of model accuracy and standard adversarial training as baseline of adversarial accuracy. The result is summarized in table \ref{tab:toy}. 

\begin{table}[h]
    \centering
    \begin{tabular}{c|c|c|c|c}
    \hline
          \quad   & Bayes & SVM   &  Adv.Tr  & AWT \\
           acc.   & 0.999 & 0.995 &  0.995   & 0.995 \\
           ang.   & 0.0   & 0.555 &  0.202   & 0.118 \\
           cos.   & 1.0   & 0.850 &  0.980   & 0.993 \\
           rob.   & 0.843 & 0.714 &  0.823   & 0.838\\
    \hline
    \end{tabular}
    \vspace{2mm}
    
    \caption{'cos.' represents cosine of angle. 'rob.' represents adversarial accuracy.}
    \label{tab:toy}
\end{table}

Table \ref{tab:toy} shows that SVM reaches comparable model accuracy as Bayes model but with a large deflection angle ($0.555\approx 31.8^\circ$), this simulates the case that natural training can reach high accuracy but fails to be adversarial robust. Adversarial training reaches the same model accuracy but also improved adversarial accuracy (0.823), which is very close to the one of Bayes classifier. This is also reflected by the deflection angle ($0.202\approx 11.6^\circ$), which shows a significant decrease when compared with the deflected angle of SVM. Our method (AWT) gets slightly better performance than the dense adversarial training with only 10\% of the weights left. 

\subsection{More Real Data Experiments}
We first present the detailed experimental configurations and then provide the experimental results omited in the main text due to the space limitation. 
\paragraph{Experimental Configuration}
We conduct experiments on standard datasets, including MNIST \citep{lecun1998gradient}, CIFAR-10 and CIFAR-100 \citep{krizhevsky2009learning}. All experiments are performed in JAX \citep{jax2018github}, together with the neural-tangent library \citep{neuraltangents2020}. 

For the neural networks, in this paper, we evaluate our proposed method on two networks: MLP, 6-layer CNN and VGG16. The MLP is comprised of two hidden layers containing 300 and 100 units with rectified linear unit(ReLUs) followed with a linear output layer. The CNN has two convolutional layers with 32 and 64 channels, respectively. Each convoluational layer is followed by a max-pooling layer and the final two layers are fully connected with 1024 and 10/100 hidden nodes, respectively.  For the experiments on VGG16, we explore the possibility of scaling up our method on large-size modern neural networks by using the technique such as sampling on the MTK matrix.  
\begin{figure}[htb!]
    \centering
    \includegraphics[width=0.9\textwidth]{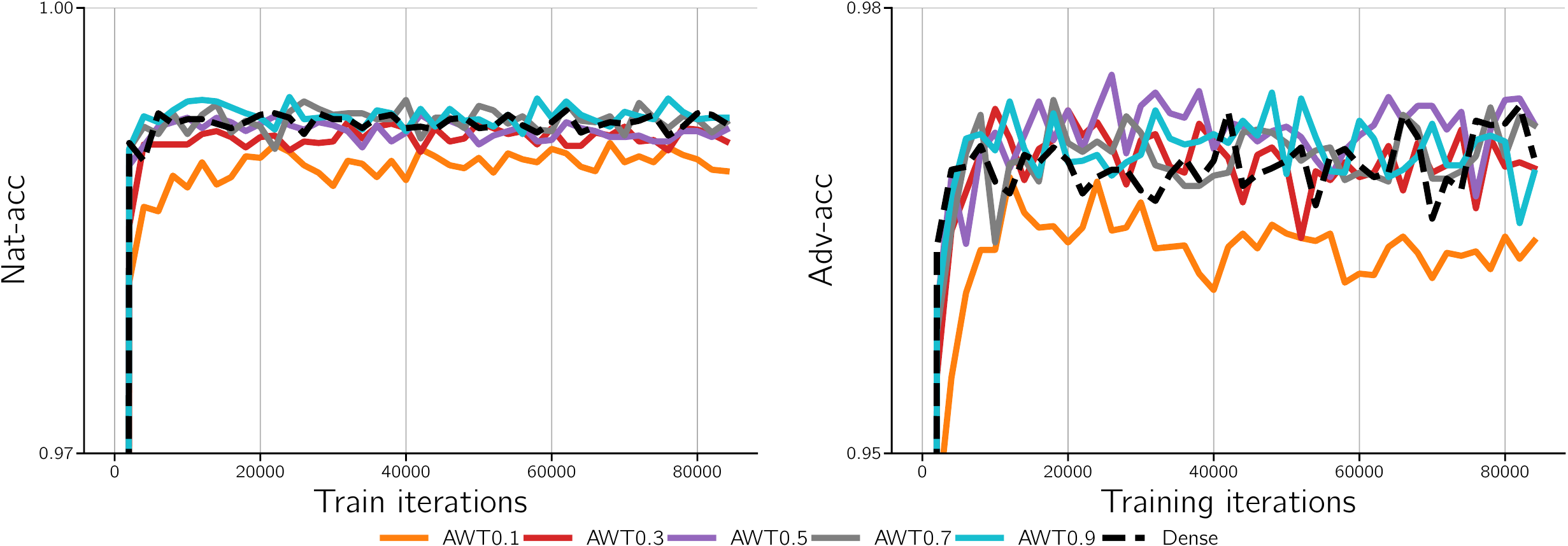}
    \caption{Test accuracy on natural and adversarial examples of CNN trained on MNIST. The density varies in $\{0.1,0.3,0.5,0.7,0.9\}$.}
    \label{fig:acc-mnist-cnn-01-09}
\end{figure}

We evaluate the performance of our method on different density levels. To be precise: 1) For the experiments on MNIST with MLP, we vary the density level in $\{0.018,0.036,0.087,0.169,0.513\}$; For the experiments with CNN, the density level is set to be $[0.05,0.1,0.2,0.3, \ldots, 0.9]$; 3) Since MNIST is much easier to be classified compared with CIFAR10 and CIFAR100, we also evaluate the performance of our method on MNIST with CNN at the density levels $\{0.01, 0.02, \ldots, 0.05\}$.   Each experiment is comprised of two phases. For the experiments on MLP and CNN, in phase one, we run Algorithm 1 for 20 epochs to find the adversarial winning ticket. The weight $\gamma$ of the kernel distance in Eqn.(10) is choosen to be $1e-3$. In phase two, we adversarially train the winning ticket from the original initialization with the cross entropy loss for 100 epochs by using $L_\infty$-PGD attack. The $\epsilon$ in PGD attack in both phase one and two is set to be 0.3 and 8/255 in the experiments on MNIST and CIFAR-10/100, respectively.  In both of these two phases, we adopt adopt Adam \citep{kingma2014adam} to solve the corresponding optimization problem. The learning rate is set to be $5e-4$ and $1e-3$ in phase one and phase two, respectively. The batch size is 64. For the experiments on VGG16, we run phase one for 10 epochs and phase two for 20 epochs. Other settings such as $\epsilon$ are the same as the experiments on MLP and CNN.

\paragraph{Experiments in Introduction} For the experimental results given in the introduction section, we adopt the above MLP network and do nature training. We compare the dynamic preserving abilities of NTT and DNS in pruning with the pruning rate being $0.02$. NTT prunes the network at initialization, while DNS prunes the network during training. In both NTT and DNS, we prune the network in 20 epochs with batch size being 64. And then we fine tune the obtained sparse network for 50 epochs. 

\paragraph{An Ablation Study} To show whether the training dynamics is preserved, instead of only looking at the kernel distance and target distance, we adopt a technique named network grafting \citep{gu2020characterize} to verify whether dynamics are preserved. The idea is if two networks have same dynamics, then one can be grafted/connected onto the other at each layer at any epoch of training without significant error increase. We give the result on MNIST with CNN in Figure \ref{fig:graft}, where the error increases are very small, especially when the densities are low. This verifies the claim. 

\begin{figure}[htb!]
    \centering
    \includegraphics[scale=2]{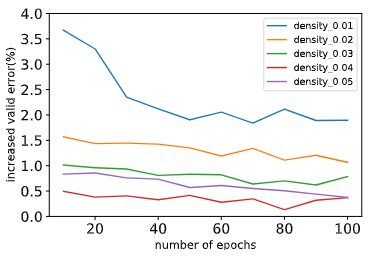}
    \caption{Graft results.}
    \label{fig:graft}
\end{figure}

\paragraph{More Results}
Figure \ref{fig:kernel-target-distance-AWT-cnn-mnist}(a) and (b) are the statistics of AWT, i.e, the kernel and target distances, in the procedure of searching masks on MNIST with cnn over density levels of $\{0.05,0.1,0.2,0.3,\ldots, 0.9\}$. Figure \ref{fig:kernel-target-distance-AWT-cnn-mnist}(c) is the adversarial training and test accuracy curves in the training process (phase two). To make the curves not too crowded, we omit the curves at the density levels of $\{0.2,0.4,0.6,0.8,0.9\}$. We can see that the target and kernel distances can decrease quickly in phase one and training dynamic of the winning ticket in phase two would become closer to the dense network when the density level increases. This is consistent with our theoretical analysis. 
\begin{figure}[htb!]
    \centering
    \includegraphics[width=0.9\textwidth]{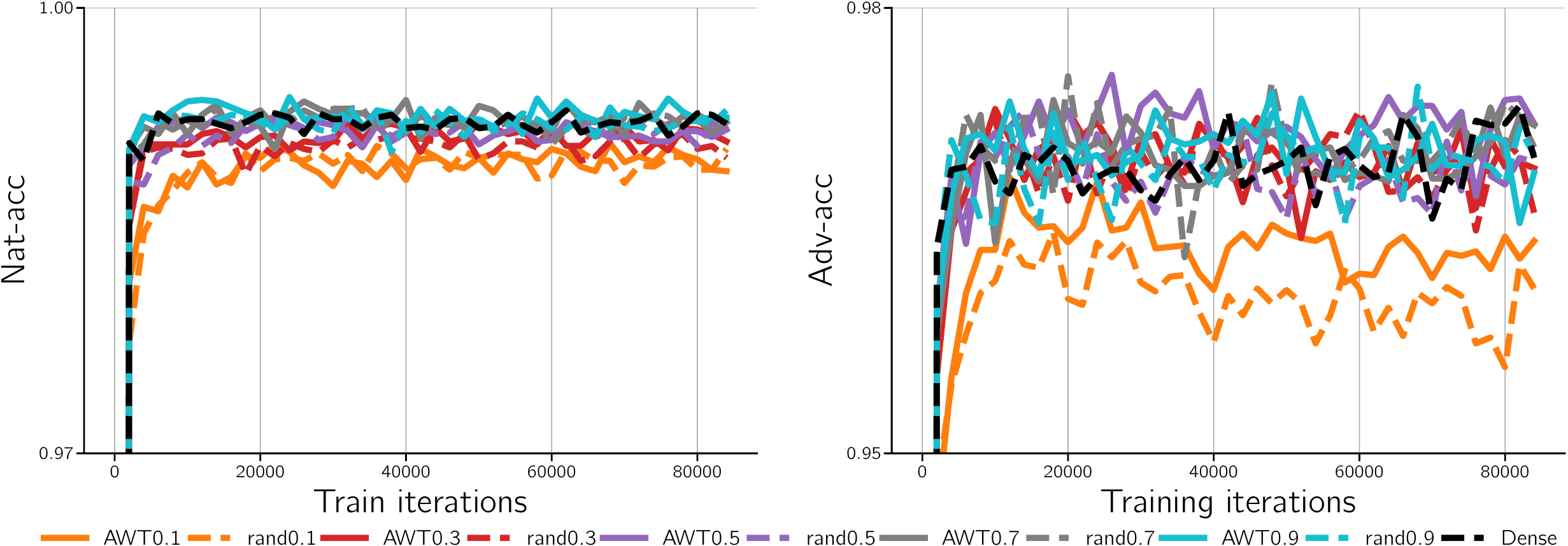}
    \caption{Natural and adversarial test accuracy of the models trained from AWT and random structure on MNIST with CNN. The density varies in $\{0.1,0.3,0.5,0.7,0.9\}$.}
    \label{fig:acc-mnist-cnn-01-09-AWT-rand}
\end{figure}

\begin{figure}[htb!]
    \centering
    \includegraphics[width=0.9\textwidth]{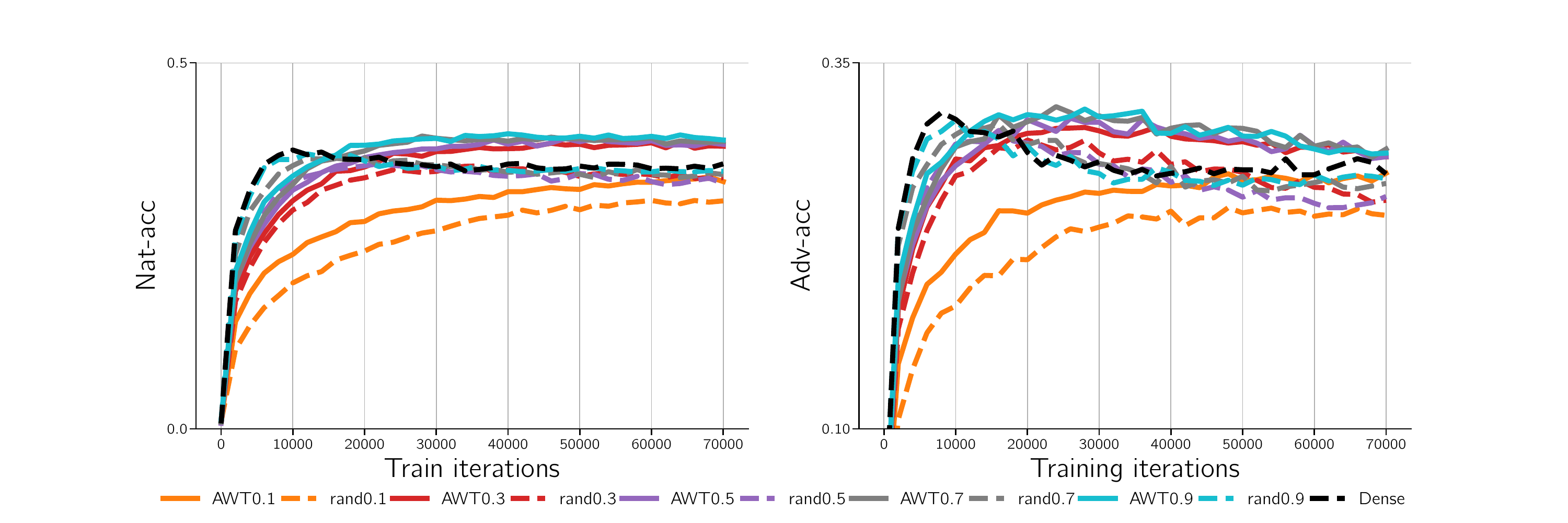}
    \caption{Natural and adversarial test accuracy of the models trained from AWT and random structure on CIFAR100 with CNN. The density varies in $\{0.1, 0.3,0.5,0.7,0.9\}$.}
    \label{fig:acc-cifar100-cnn-01-09-AWT-rand}
\end{figure}
Figure \ref{fig:acc-mnist-cnn-01-09} presents the natural and adversarial testing accuracy of CNN trained on MNIST with the density level varies in $\{0.1,0.2,\ldots,0.9\}$. We can see that when the density level is larger than 0.2, the accuracy is very close to the dense mode. The reason could be that when the density level is larger than 0.2, the model begins to be overparameterized. This can also be seen in Figure \ref{fig:acc-mnist-cnn-01-09-AWT-rand}. That is when the density level is larger than 0.2, there is even no significant difference between the winning ticket an the random structure. That's why we give the results with the density level varying in $\{0.01,0.02,\ldots, 0.05\}$ in the main text.

Figure \ref{fig:acc-cifar100-cnn-01-09-AWT-rand} shows the performance of the models trained on CIFAR100 from our winning ticket and the random structure. We can see that after training, our winning ticket has significantly higher natural and adversarial test accuracy than that of the random structure. In this experiment, all the models cannot achieve the comparable test accuracy on natural examples as ResNet18 reported in the existing studies. The reason is that our model is a 6-layer CNN, whose capacity is much smaller than ResNet18.

\end{appendix}

\end{document}